%% file: full-paper.tex
\documentclass[11pt]{article}

\usepackage{tikz}
\usepackage{graphicx,wrapfig,lipsum}   

\usepackage{xspace}

\usepackage{amsmath,amssymb}
\usepackage{amsthm}
\usepackage{mathtools}
\usepackage[lined,boxed,ruled,norelsize,algo2e,linesnumbered]{algorithm2e}
\usepackage{url}
\usepackage{fullpage}
\usepackage{makeidx}
\usepackage{enumerate}
\usepackage[top=1in, bottom=1in, left=0.75in, right=0.75in]{geometry}
\usepackage{graphicx,float,psfrag,epsfig,caption}
\usepackage{bbm}

\usepackage{hyperref}
\usepackage{epstopdf}
\usepackage{color}
\usepackage{xr}
\usepackage{subfig}
\usepackage{caption}
\usepackage[utf8]{inputenc}

\usepackage{scalerel,stackengine}

\usepackage{enumitem}

\stackMath
\newcommand\reallywidehat[1]{%
\savestack{\tmpbox}{\stretchto{%
  \scaleto{%
    \scalerel*[\widthof{\ensuremath{#1}}]{\kern.1pt\mathchar"0362\kern.1pt}%
    {\rule{0ex}{\textheight}}%WIDTH-LIMITED CIRCUMFLEX
  }{\textheight}% 
}{2.4ex}}%
\stackon[-6.9pt]{#1}{\tmpbox}%
}
\usepackage[mathscr]{euscript}

\DeclareSymbolFont{rsfs}{U}{rsfs}{m}{n}
\DeclareSymbolFontAlphabet{\mathscrsfs}{rsfs}

\numberwithin{equation}{section}

\newtheoremstyle{myexample} % name
    {\topsep}                    % Space above
    {\topsep}                    % Space below
    {\rm }                   % Body font
    {}                           % Indent amount
    {\bf }                   % Theorem head font
    {.}                          % Punctuation after theorem head
    {.5em}                       % Space after theorem head
    {}  % Theorem head spec (can be left empty, meaning normal)

\newtheoremstyle{myremark} % topsep
    {\topsep}                    % Space above
    {\topsep}                    % Space below
    {\rm}                        % Body font
    {}                           % Indent amount
    {\bf}                        % Theorem head font
    {.}                          % Punctuation after theorem head
    {.5em}                       % Space after theorem head
    {}  % Theorem head spec (can be left empty, meaning normal)

\newtheorem{claim}{Claim}[section]
\newtheorem{lemma}[claim]{Lemma}

\newtheorem{theorem}{Theorem}

\newtheorem{corollary}[claim]{Corollary}
\newtheorem{definition}[claim]{Definition}

\theoremstyle{myremark}
\newtheorem{remark}{Remark}[section]

\theoremstyle{myremark}

\theoremstyle{myexample}

\newcommand{\wt}{\widetilde}
\newcommand{\wh}{\widehat}
\newcommand{\mcl}{\mathcal}
\newcommand{\norm}[1]{\left\lVert#1\right\rVert}
\newcommand{\Gap}{\textsf{Gap}}

\definecolor{darkgreen}{rgb}{0.0, 0.5, 0.0}
\newcommand{\mscomment}[1]{\noindent{\textcolor{darkgreen}{\textbf{\#\#\# MS:} \textsf{#1} \#\#\#}}}

\input{commands.tex}

\usepackage{tikz}
\usepackage{enumitem}
\usetikzlibrary{calc}
\usetikzlibrary{patterns}

\tikzstyle{every node}=[circle, draw, fill=black!50, inner sep=0pt, minimum width=4pt]
\tikzstyle{rouge}=[circle, draw, fill=red, inner sep=0pt, minimum width=6pt]
\tikzstyle{bleu}=[circle, draw, fill=blue, inner sep=0pt, minimum width=6pt]
\tikzstyle{petitrouge}=[circle, draw, fill=red, inner sep=0pt, minimum width=4pt]
\tikzstyle{petitbleu}=[circle, draw, fill=blue, inner sep=0pt, minimum width=4pt]
\tikzstyle{txt}=[draw=none, fill=none]
\tikzstyle{white}=[circle, draw, fill=black!0, inner sep=0pt, minimum width=4pt]

\definecolor{ffqqqq}{rgb}{1,0,0}
\definecolor{qqzzqq}{rgb}{0,0.6,0}
\definecolor{qqqqff}{rgb}{0,0,1}
\definecolor{uuuuuu}{rgb}{0.27,0.27,0.27}

\def\L{\mathcal{L}}

\def\Feas{\mathrm{Feas}}
\def\range{\mathrm{range}}
\def\gap{\mathrm{gap}}
\def\CHILDREN{\mathrm{CHILDREN}}

\def\ROOT{\mathrm{ROOT}}

\title{The Pareto Frontier of Instance-Dependent Guarantees in Multi-Player Multi-Armed Bandits with no Communication}
\author{
Allen Liu\thanks{Email: \texttt{cliu568@mit.edu}. Part of this work was completed at Microsoft Research and supported by an NSF Graduate Research Fellowship, a Hertz Foundation Fellowship, and NSF CAREER Award CCF-1453261 and NSF Large CCF1565235.} \and Mark Sellke\thanks{Email: \texttt{msellke@stanford.edu}. Part of this work was completed at Microsoft Research and supported by an NSF Graduate Research Fellowship, a Stanford Graduate Fellowship, and NSF grant CCF-2006489.}}
\date{}

\begin{document}

\maketitle

% \begin{abstract}%
%  We study the stochastic multi-player multi-armed bandit problem.  In this problem, $m$ players cooperate to maximize their total reward from $K > m$ arms. However the players cannot communicate and are penalized (e.g. receive no reward) if they pull the same arm at the same time.  We ask whether it is possible to obtain optimal instance-dependent regret $\wt{O}(1/\Delta)$ where $\Delta$ is the gap between the $m$-th and $m+1$-st best arms.  Such guarantees were recently achieved by \cite{pacchiano2021instance, huang2021towards} in a model in which the players are able to implicitly communicate through intentional collisions. We show that with no communication at all, such guarantees are, surprisingly, not achievable.  In fact, obtaining the optimal $\wt{O}(1/\Delta)$ regret for some regimes of $\Delta$ necessarily implies strictly sub-optimal regret in other regimes.  Our main result is a complete characterization of the Pareto optimal instance-dependent trade-offs that are possible with no communication. Our algorithm generalizes that of \cite{bubeck2021cooperative} and enjoys the same strong no-collision property, while our lower bound is based on a topological obstruction and holds even under full information.
% \end{abstract}

\begin{abstract}%
 We study the stochastic multi-player multi-armed bandit problem.  In this problem, there are $m$ players and $K > m$ arms and the players cooperate to maximize their total reward. However the players cannot communicate and are penalized (e.g. receive no reward) if they pull the same arm at the same time.  We ask whether it is possible to obtain optimal instance-dependent regret $\wt{O}(1/\Delta)$ where $\Delta$ is the gap between the $m$-th and $m+1$-st best arms.  Such guarantees were recently achieved by \cite{pacchiano2021instance, huang2021towards} in a model in which the players are able to implicitly communicate through intentional collisions.

 Surprisingly, we show that with no communication at all, such guarantees are not achievable.  In fact, obtaining the optimal $\wt{O}(1/\Delta)$ regret for some values of $\Delta$ necessarily implies strictly sub-optimal regret for other values.
 Our main result is a complete characterization of the Pareto optimal instance-dependent trade-offs that are possible with no communication. Our algorithm generalizes that of \cite{bubeck2021cooperative}. As there, our algorithm succeeds even when feedback upon collision can be corrupted by an adaptive adversary, thanks to a strong no-collision property. Our lower bound is based on topological obstructions at multiple scales and is completely new.
\end{abstract}

\tableofcontents

\input{intro-v2.tex}

\input{lower-bound-v2.tex}
\input{alg-outline.tex}
\input{partition.tex}
\input{analysis-main.tex}

\input{regret.tex}

\section*{Acknowledgement}

We thank S{\'e}bastien Bubeck for several helpful discussions and for encouraging us to work on this problem.

{
\bibliographystyle{alpha}
\bibliography{all-bib}
}

\end{document}

%% file: commands.tex
\newcommand{\bea}{\begin{eqnarray}}
\newcommand{\eea}{\end{eqnarray}}
\newcommand{\<}{\langle}
\renewcommand{\>}{\rangle}
\newcommand{\E}{{\mathbb E}}

\def\poly{\mathrm{poly}}

\def\eps{{\varepsilon}}

\def\cT{{\mathcal T}}

\def\E{{\mathbb E}}

\def\<{\langle}
\def\>{\rangle}

\def\cP{{\mathcal P}}

\def\wt{\widetilde}

\def\b0{{\boldsymbol{0}}}
\def\bp{{\boldsymbol{p}}}
\def\p{{\boldsymbol{p}}}
\def\q{{\boldsymbol{q}}}

\def\cA{{\mathcal A}}

\def\cP{{\mathcal P}}

\newcommand{\x}{\mathbf{x}}
\newcommand{\y}{\mathbf{y}}

\renewcommand{\b}{\mathbf{b}}

\newcommand{\R}{\mathbb{R}}

\def\lt{\left}
\def\rt{\right}
\def\Ind{\mathbbm{1}}

\def\vD{\vec{\Delta}}

%% file: intro-v2.tex
\section{Introduction}

We consider the stochastic multi-player multiarmed bandit problem with $m$ players and $K>m$ arms. An instance of this problem is described by the mean rewards $\bp = (p_1,\cdots, p_K) \in [0, 1]^K$, which are unknown to the players. For convenience we assume that $p_1\geq p_2\geq\cdots\geq p_K$.  At each time step $t=1, \hdots, T$, each player $X \in [m]$ chooses an action $i_t^X \in [K]$.  They observe the random variable $Y_t^X$ which has $\Pr[Y_t^X = 1] = p_{i_t^X} $ and $\Pr[Y_t^X = 0] = 1 -  p_{i_t^X} $.  However, they \emph{receive} the reward
%Let $(Y_t(i))_{1 \leq i \leq K, 1 \leq t \leq T}$ be a sequence of independent random variables with $\mathbb P(Y_t(i) = 1) = p_i$ and $\mathbb P(Y_t(i) = 0) = 1-p_i$. 
\begin{equation}
\label{eq:reward}
    r_t(X)=Y_t^X\cdot \Ind_{i_t^X\neq i_t^Y~\forall~ Y\neq X}.
\end{equation}
In other words, all players that pull the same arm observe independent realizations from the corresponding arm but their actual reward is $0$.  We define the expected regret $R_T=R_T(\cA;\bp)$ of an algorithm $\cA$ by
\begin{equation}
\label{eq:regret-defn}
    R_T(\cA;\bp) = \E\lt[T \cdot (p_1+\dots+p_m) - \sum_{t=1}^T \sum_{X = 1}^m r_t(X)\rt].
\end{equation}
We assume the players have access to shared randomness (which can be visible also to the adversary). The players may coordinate ahead of time but receive no online feedback aside from $Y_t^X$. As usual we aim to minimize the regret -- note that the implicit benchmark in \eqref{eq:regret-defn} is based on the sum $p_1+\dots+p_m$ of the top $m$ mean rewards, which is the maximum possible average reward if $\bp$ is known.

\subsection{No Communication vs Implicit Communication}

The core feature of the the multi-player bandit problem is that collisions between different players are costly and must be avoided, e.g. lead to zero reward. Yet this high-level description has admitted at least three precise formulations that have all been extensively studied, which differ in the extent to which collisions can be detected by the players.  These formulations are listed below in increasing order of difficulty.
\begin{enumerate}
    \item \textbf{Strongly} detectable collisions: players are explicitly told when they have collided (e.g. \cite{AM14}).
    \item \textbf{Weakly} detectable collisions: the players are not explicitly told about the collision, but observe the same reward $r_t(X)$ they receive (instead of observing $Y_t^X$ as described above). This model was introduced in \cite{bonnefoi2017multi}, see also \cite{lugosi2021multiplayer}.
    \item \textbf{Undetectable} collisions: 
    collisions have \textbf{no effect} on the feedback received. This is the main model we focus on: the observed rewards are $Y_t^X$ as described above.
\end{enumerate} 

% We consider the cooperative multi-player version of the stochastic multi-armed bandit problem.  This problem is motivated by cognitive radio applications where users are trying to transmit messages through various channels and was first studied in \cite{LJP08,LZ10, AMTS}.  Since then there has been a flurry of work \cite{AM14,RSS16,alatur2020multi,BLPS20, BP18,lugosi2021multiplayer,BLPS20,huang2021towards,pacchiano2021instance}.  
It is unsurprising that formulation $1$ allows the players to communicate using a small number of intentional collisions. Indeed, several works such as \cite{AM14,RSS16,alatur2020multi,BLPS20} have implemented intricate communication protocols under this model. However by communicating, these algorithms essentially destroy the decentralization motivating the original problem.

Results for the second formulation were obtained in \cite{BP18,lugosi2021multiplayer,BLPS20,huang2021towards,pacchiano2021instance}, including regret $O\left(Km\log T\sqrt{T}\right)$ in \cite{lugosi2021multiplayer}. While this is a more realistic model,  it turns out that it still allows the players to implement highly intricate communication protocols. 
%as the weak detectability essentially means that the players are pushed away from collisions.
Indeed, the strategies in \cite{huang2021towards,pacchiano2021instance} use the idea of repeatedly playing an arm for a long period of time, meaning that other players who choose that arm receive a decreased reward; this turns out to suffice for rather general communication. 
%Such strategies were also used previously in e.g.\cite{lugosi2021multiplayer,shi2020decentralized}, but with implementations that required the smallest mean reward to be bounded away from $0$. 

% In light of this, we ask what happens when we rule out communication protocols of any kind. 
In light of this, it is natural to ask what happens when we rule out communication protocols of any kind. 
Recently \cite{bubeck2020coordination,bubeck2021cooperative} showed that it is possible to achieve regret $O(\poly(K)\sqrt{T\log T})$ with no communication, and in fact with \textbf{no collisions at all} (with high probability, say $1-T^{-2}$).

%Their key insight is that it is possible to agree beforehand on a randomized partition of the state space of arm estimates enjoying desirable stability properties. 

In single player stochastic bandits, one can go beyond just $O(\sqrt{T})$ and obtain optimal instance-dependent rates scaling with the inverse of the gap between the best arm and the rest.  Several works such as \cite{besson2018multi,BP18,shi2020decentralized} have aimed at such guarantees in the multi-player setting; note that the gap is now defined as $\Delta=p_m-p_{m+1}$. Recently, \cite{huang2021towards,pacchiano2021instance} have obtained optimal regret $O\left(\frac{\poly(K)\log T}{\Delta}\right)$ using algorithms that rely heavily on communication.
% However their algorithms rely extensively on communication in order to synchronize exactly when to eliminate suboptimal arms
% By contrast, it is much more difficult to accomplish this without any communication at all, as it becomes unavoidable that players will decide to eliminate arms at different times. 
In this work, we ask:

\[
    \textit{Is it possible to achieve regret $\widetilde O(1/\Delta)$ for multi-player bandits without communication of any kind?}
\]

%Of course, in stochastic bandit problems it is also desirable to obtain instance-dependent rates scaling with the inverse of the gap, which is simply $\Delta=\p_m-\p_{m+1}$ in the multi-player setting. 

\subsection{Main Result}

We show that in fact, the lack of implicit communication completely alters the landscape of what is possible for instance-dependent trade-offs.  Furthermore, we completely characterize the optimal trade-offs for instance-dependent regret across different values of $\Delta$ up to $\poly(K,\log T)$ factors.  Formally, given $\bp$, we define the gap 
\[
    \Delta=\Delta(\bp)=p_m-p_{m+1}
\]
and we define the gap-dependent regret of an algorithm $\cA$ by
\[
    R_{T,\Delta}(\cA)\equiv \sup_{\Delta(\bp)\geq \Delta} \E[R_T(\cA;\bp)].
\]
Our main results are stated below, with $\Omega(\cdot),O(\cdot)$ indicating only absolute constant factors.

\begin{theorem}
\label{thm:main-lower}
For any $\cA$, there exists such a sequence $\vD = (\Delta_0, \dots , \Delta_J)$ with
\[
    1=\Delta_0>\Delta_1>\dots>\Delta_J= T^{-1/2}
\]
such that
\begin{equation}
\label{eq:main-lower-bound}
    R_{T,\Delta}(\cA)\geq  \Omega\lt(\frac{1}{\Delta_j\Delta_{j+1} \cdot  \log T}\rt),\quad \Delta \in (\Delta_{j+1},\Delta_j].
\end{equation}
\end{theorem}

\begin{theorem}\label{thm:main-upper}
Consider a decreasing sequence $\vD = (\Delta_0, \dots , \Delta_J)$ given by 
\[
    1=\Delta_0>\Delta_1>\dots>\Delta_J= T^{-1/2}.
\]
Then there exists an algorithm $\cA$ such that
\begin{equation}
\label{eq:main-upper-bound}
    R_{T,\Delta}(\cA)\leq  O\lt(\frac{m K^{6} \log^2 (KT) }{\Delta_j \Delta_{j+1}}\rt),\quad \Delta \in (\Delta_{j+1},\Delta_j].
\end{equation}
Furthermore with probability at least $1-\frac{1}{T}$, there are no collisions between any players at any time.
\end{theorem}

Together, Theorem~\ref{thm:main-lower} and Theorem~\ref{thm:main-upper} identify the Pareto-optimal functions $R_{T,\Delta}$ up to $\poly(\log(T),K)$ factors. As special cases, we deduce that both $R_{T,\Delta}\leq \wt O(\sqrt T)$ and $R_{T,\Delta}\leq \wt O(\Delta^{-2})$ are essentially unimprovable.  Furthermore, while it is possible to obtain regret $\widetilde O\left(\frac{1}{\Delta}\right)$ for any \textbf{fixed} value of $\Delta$, it is \textbf{not} possible to obtain such a guarantee for even two separated values $\Delta_1 \gg \Delta_2$ simultaneously.  Our algorithm achieving \eqref{eq:main-upper-bound} builds on the aforementioned work of \cite{bubeck2020coordination,bubeck2021cooperative}, while our lower bound is completely new.

%while it is possible to adapt optimally to a \emph{fixed} $\Delta$ known in advance, the same is not true even for just two separated values $\Delta_1\ll \Delta_2$.  Furthermore, 

\begin{figure}[!h]
\centering
  \includegraphics[width=0.8\linewidth]{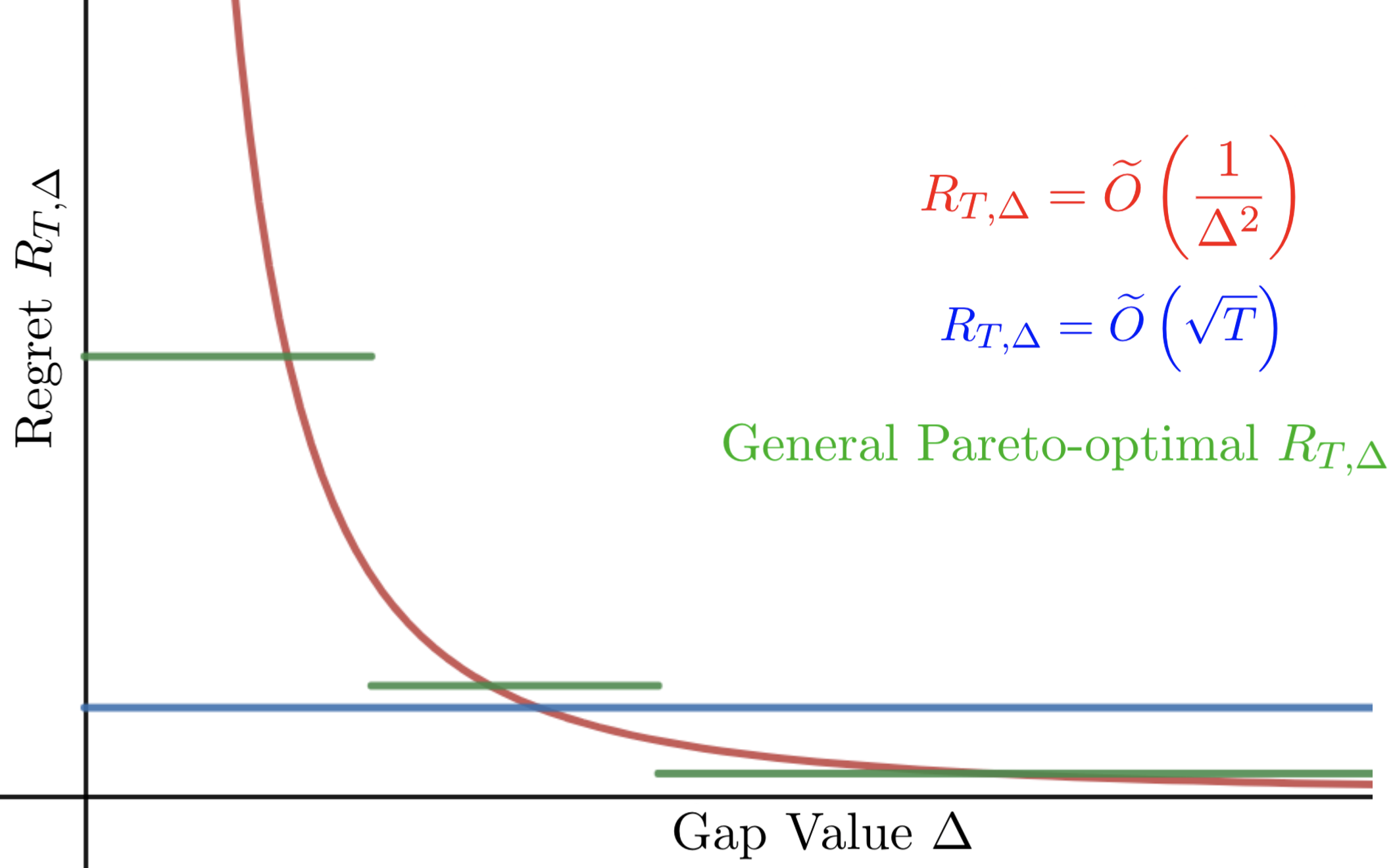}
  \caption{Several examples of Pareto-Optimal regret curves are shown as functions of the gap $\Delta$.}
\end{figure}

\begin{remark}
It is possible to give yet more pessimistic problem formulations for the multi-player bandit. For instance, one may posit that an \textbf{adaptive adversary} can corrupt the feedback of any player involved in a collision. 
Because players following our main algorithm never collide at all, the regret guarantees stated in Theorem~\ref{thm:main-upper} remain valid in any problem formulation as long as genuine feedback is received whenever no collision occurs. Given the matching lower bounds of Theorem~\ref{thm:main-lower}, it follows that no formulation of the problem is more difficult than undetectable feedback from a regret viewpoint, even with gap dependence taken into account (up to lower order factors). 
We find it striking and surprising that malicious, adaptive adversaries can do no better than to report the underlying Bernoulli variables $Y_t^X$ as feedback upon collision.
\end{remark}

\subsection{Consequences}

The next few corollaries are immediate consequences of Theorem~\ref{thm:main-lower}. They illustrate the robustness of our Pareto-optimality result. For instance Corollary~\ref{cor:sqrtT} constrains the performance of algorithms that \textbf{almost} achieve the minimax-optimal regret. Corollary~\ref{cor:Delta1+2} illustrates that aggressively optimizing $R_{T,\Delta_1}$ for a fixed $\Delta_1$ leads to very large regret for slightly smaller gaps $\Delta_2$. Indeed, a practical interpretation of Theorem~\ref{thm:main-lower} is that without implicit communication, one should not try to improve regret guarantees for large $\Delta$ unless one is very confidence that $\Delta$ is in fact large. 

\begin{corollary}
\label{cor:sqrtT}
Any $\cA$ achieving $\sup_{\bp}R_T(\cA;\bp)\leq T^{\frac{1}{2}+\eps}$ must satisfy $R_{T,\Delta}(\cA) \geq \wt\Omega(T^{1/2})$ for all $\Delta\leq T^{-\eps}.$

\end{corollary}

\begin{corollary}
\label{cor:Delta-2}
Any $\cA$ achieving 
\[
    R_T(\cA;\bp)\leq \frac{T^{\eps}}{\Delta(\bp)^2}
\]
for all $\bp$ must also satisfy $R_{T,\Delta}(\cA) \geq \wt\Omega(\Delta^{-2})$ for all $\Delta\leq T^{-\eps}.$

\end{corollary}

\begin{corollary}
\label{cor:Delta1+2}
Let $(\Delta_1,\Delta_2)\in [T^{-1/2},1]^2$ satisfy $\Delta_1\geq \Delta_2\cdot T^{\eps}$. Any $\cA$ achieving 
\[
    R_{T,\Delta_1}(\cA)\leq \wt O(\Delta_1^{-1})
\]
must also satisfy 
\[
    R_{T,\Delta_2}(\cA)\geq \wt O\lt(\frac{1}{\Delta_1\cdot\Delta_2}\rt).
\]
\end{corollary}

%% file: lower-bound-v2.tex
\section{Lower Bound}\label{sec:lowerbound}

In this section, we prove a lower bound on the regret from which Theorem~\ref{thm:main-lower} will follow.  The main lemma is stated below.

\begin{lemma}\label{lem:regret-at-some-scale2}
Fix $K > m \geq 2$.  For any algorithm $\cA$ and time horizon $T$, there exists $\Delta \in [1/\sqrt{T}, 1]$ and a problem instance given by $\bp$ with $\Delta(\bp) \geq \Delta$ such that
\[
R_T(\cA ; \bp) \geq \Omega\left( \frac{\sqrt{T}}{ \Delta \log T} \right) \,.
\]
\end{lemma}

We first show how Theorem~\ref{thm:main-lower} follows easily from Lemma~\ref{lem:regret-at-some-scale2}. 
\begin{proof}[Proof of Theorem~\ref{thm:main-lower}]
Set $\Delta_J = T^{-1/2}$.  Apply Lemma~\ref{lem:regret-at-some-scale2} to find some $\Delta$ and set $\Delta_{J-1} = \Delta$.  By the guarantees of Lemma~\ref{lem:regret-at-some-scale2}, there is some instance $\p$ with $\Delta(\p) \geq \Delta_{J-1}$ such that 
\[
R_{T,\Delta_{J-1}}(\cA)\geq R_T(\cA ; \p) \geq \Omega\left(\frac{1}{\Delta_{J-1}\Delta_J \cdot \log T} \right) \,.
\]
Now apply Lemma~\ref{lem:regret-at-some-scale2} again with time horizon $1/(2\Delta_{J-1})^2$.  This allows us to find some $ \Delta_{J-2}$ with $2\Delta_{J-1} \leq \Delta_{J-2} \leq 1$ such that there is some instance $\p$ with $\Delta(\p) \geq \Delta_{J-2}$ and
\[
R_{T,\Delta_{J-2}}(\cA)\geq R_T(\cA ; \p) \geq R_{1/(2\Delta_{J-1})^2}(\cA ; \p) \geq  \Omega\left(\frac{1}{ \Delta_{J-2}\Delta_{J-1} \cdot \log T} \right) \,.
\]
Repeating this process, we construct the entire sequence $\vD = (\Delta_0, \dots , \Delta_J)$.  Note that the ratio between adjacent elements is at least $2$ so we can choose $J$ appropriately (with $J \leq \log_2 T$) and ensure that the process terminates with $\Delta_0 = 1$.  The lower bounds on $R_{T,\Delta_j}(\cA)$ for other $j$ follow exactly as above.
\end{proof}

The remainder of this section is devoted to proving Lemma~\ref{lem:regret-at-some-scale2}.

\subsection{Full Information Setting}

In fact, we will prove that the lower bound in Lemma~\ref{lem:regret-at-some-scale2} (and thus also Theorem~\ref{thm:main-lower}) holds even in a stronger setting where the players observe independent realizations of the rewards of \textbf{all} arms (not just the arm that they selected).  Formally,
\begin{definition}\label{def:full-info}
In the full information setting, the instance is described by the mean rewards $\bp = (p_1, \dots , p_K)$.  At each timestep $t = 1, \dots , T$, each player $X$ chooses an action $i_t^X \in [K]$.  They observe a reward vector $V_t^X$ where the entries of $V_t^X$ are independent Bernoulli random variables with means $p_1, \dots , p_k$ respectively (and the observations of all of the players are independent). 
\end{definition}

It is clear that the full information setting gives the players more information so any lower bounds we prove in the full information setting immediately extend to the bandit setting.  In the remainder of this section, we will work in the full information setting.

\subsection{Reduction to Mean-based Strategies}

First, we show that it suffices to consider when the actions of all players depend only on their empirical estimates of the mean rewards of each arm.  Note that the joint strategy of the players may be randomized so we can think of a joint strategy at timestep $t$ as a distribution over $m$-tuples of functions $(g_1^t, \dots , g_m^t)$ where $g_i^t: \{ \{0,1 \}^t \}^K  \rightarrow \{1,2, \dots , K \} $.  A full strategy then consists of one such strategy at each timestep $t$.

We will prove that in fact, it suffices to consider a simpler family of strategies where each player only looks at the average reward of each arm from their observations.
\begin{definition}
For timestep $t$, player $X \in [m]$ and arm $i \in [K]$, we define 
\[
u_t^X(i) = \frac{\sum_{s = 1}^t V_s^X(i)}{t} \,.
\]
\end{definition}
\begin{definition}
We say an algorithm $\cA$ is mean-based if for all players $X$, their strategy at a timestep $t$ depends only on $(u_t^X(1), \dots , u_t^X(K))$.
\end{definition}
In other words, a mean-based strategy at a timestep $t$ is given by a distribution over $m$-tuples of functions $(f_1^t, \dots , f_m^t)$ where $f_i^t: \R^k  \rightarrow \{1,2, \dots , k \}$.
\begin{claim}\label{claim:means-suffice}
For any algorithm $\cA$, there is a mean-based algorithm $\cA'$ such that for all time horizons $T$ and all instances $\bp$, we have
\[
R_T( \cA , \bp) =  R_T( \cA' , \bp) \,. 
\]
\end{claim}
\begin{proof}
Let $\cA'$ be the strategy where at each timestep, each player permutes their observations for each arm independently and uniformly at random and then plays according to $\cA$.  It is clear that $\cA'$ is mean-based (since observations are all either $0$ or $1$).

Now we prove that for any instance $\bp$, the strategy $\cA'$ achieves the same expected reward (and hence regret) as $\cA$.  For player $X \in [m]$ and arm $i \in [K]$, let the corresponding sequence of observations be $s_{X, i}^1, \dots , s_{X, i}^t, \dots $.  The key point is that for any possible sequences of observations $\wh{s}_{X, i}^1, \dots , \wh{s}_{X, i}^t, \dots $ and permutations $\pi_{X,i}$ on $t$ elements, we have
\[
\Pr\left[s_{X, i}^r = \wh{s}_{X, i}^r \; \forall X \in [m], i \in [K], r \in [t] \right] 
= 
\Pr\left[s_{X, i}^r = \wh{s}_{X, i}^{\pi_{X,i}(r)}  \; \forall X \in [m], i \in [K], r \in [t] \right]    \,.
\]
In other words when the players permute their observations, the resulting set of observations have the same distribution as the original observations.  This implies that the random permutations in the strategy $\cA'$ do not affect the expected reward, concluding the proof.    
\end{proof}

\subsection{Topological Obstruction}

We will first prove the lower bound in the case $m = 2, K = 3$.  We then show how to reduce an instance with arbitrary $K > m \geq 2$ to this case.  Formally, we prove:
\begin{lemma}\label{lem:regret-at-some-scale}
Consider the case $K = 3, m = 2$.  For any algorithm $\cA$ and time horizon $T$, there exists $\Delta \in [1/\sqrt{T}, 1]$ and a problem instance given by $\bp$ such that all coordinates of $\p$ are between $0.01$ and $0.99$ and $\Delta(\bp) \geq \Delta$ such that 
\[
R_T(\cA ; \bp) \geq \Omega\left( \frac{\sqrt{T}}{\Delta \log T} \right) \,.
\]
\end{lemma}
In light of Claim \ref{claim:means-suffice}, we can think of the players' combined strategy in each time-step as a distribution over mappings $\mcl{F} = (f_1, f_2)$ where $f_i: \R^3 \rightarrow \{1,2,3 \}$.  We will first treat $\mcl{F}$ as fixed and understand key structural properties that go into lower bound.  We will then account for the potential randomization in the choice of $\mcl{F}$ when we complete the proof of Lemma~\ref{lem:regret-at-some-scale}.  
%Note that it suffices to consider deterministic strategies (since a player's action does not affect observations in future rounds).

\begin{definition}\label{def:value-and-gap}
For a point $u = (u_1, u_2, u_3)\in\R^3$, define $\mcl{V}(u) =  \max_{i_1 \neq i_2} (u_{i_1} + u_{i_2})$ i.e. the sum of the two largest coordinates of $u$. 
\end{definition}

\begin{definition}
For points $u , u' \in \R^3$ and $\mcl{F}$ as above, define the gain $\mcl{G}$ of $\mcl{F}$ at the pair $u,u'$ as follows. Let $\mcl{F}(u) = ( i , j )$ and $\mcl{F}(u') = (i', j' )$.  Then
\[
\mcl{G}(\mcl{F}, u , u' ) =   \begin{cases}
 0 ~~\text{ if } i = j' \\
 u_i + u_{j'} ~~\text{ if } i \neq j' \,.
\end{cases}
\]
\end{definition}
Note that $\mcl{G}$ is the expected reward if both players play according to $\mcl{F}$, the true instance is $u$ and the first player observes empirical means $u$ while the second player observes empirical means $u'$.  The reason this notion is useful is that the two players' empirical rewards will not be the same point but will merely be close.  Thus, a good combined strategy must perform well when the two players observe different points that are merely close to each other.  The next claim illustrates the key obstruction to designing a strategy that does this.

\begin{claim}\label{claim:obstruction}
Let $\ell$ denote the line $x = y = z$ in $\R^3$.  Let $\mcl{C}$ denote a circle centered around $\ell$ in a plane orthogonal to $\ell$ of radius at least $0.1$.  Let $n > 100$ be an integer and let $Q_1, \dots , Q_n$ be $n$ evenly spaced points on this circle.   Consider points $P_1, \dots , P_n \in \R^3$ such that $\norm{P_j - Q_j}_2 \leq 0.001$ for all $j$.  Then for any function $\mcl{F}$ mapping $\R^3$ to $\{1,2, 3 \}^2$, there must exist indices $j \neq j'$ such that $|j - j'| \leq 2$  and
\begin{equation}\label{eq:loss}
 \mcl{V}(P_j) - \mcl{G}(\mcl{F}, P_j, P_{j'}) \geq 0.01
\end{equation}
where indices are taken modulo $n$.
\end{claim}

\begin{figure}[h]
\centering
\includegraphics[width=0.7\linewidth]{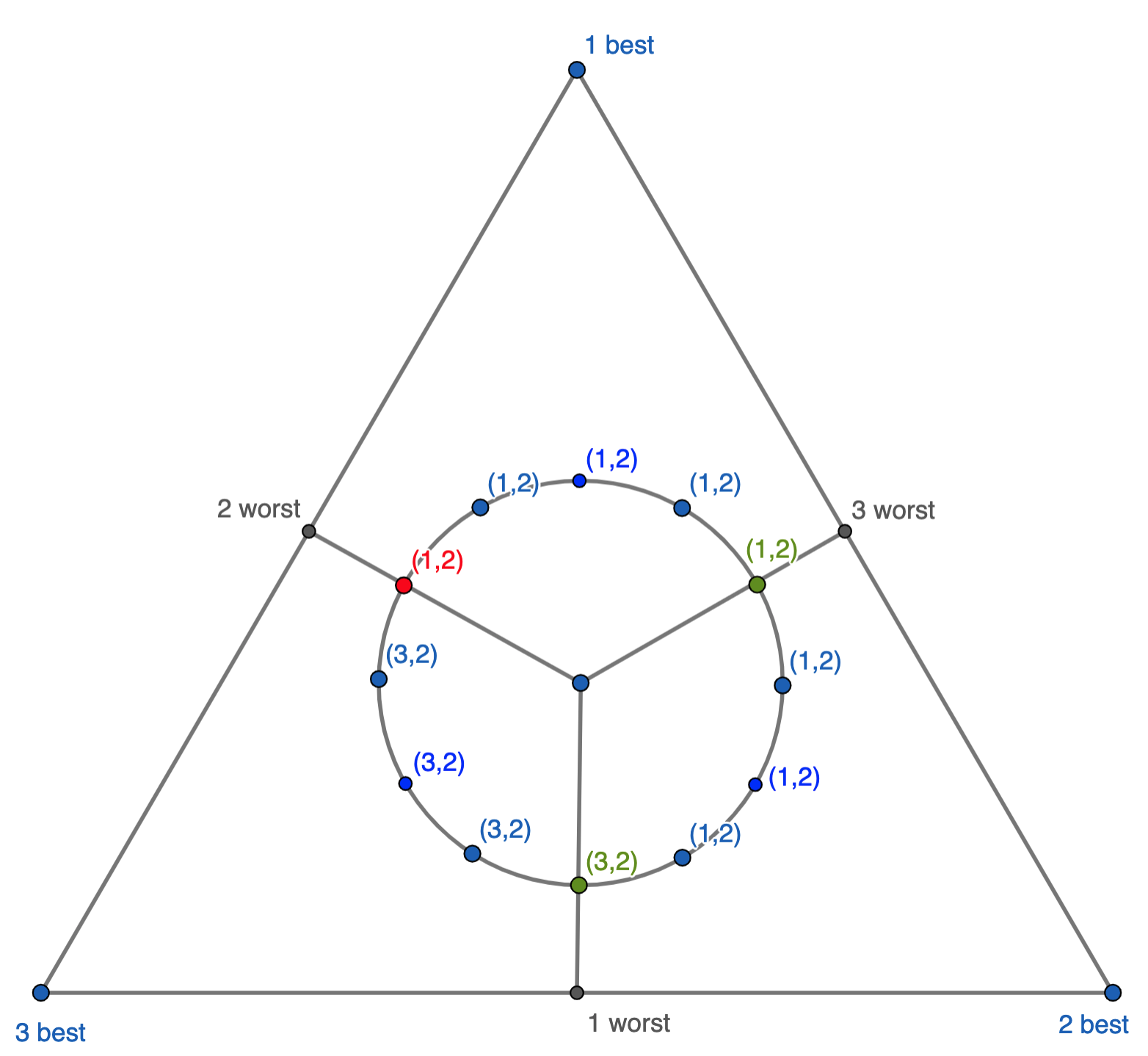}
\caption{An illustration of the key topological obstruction in Claim~\ref{claim:obstruction}.  The main idea is that the players must either collide at neighboring points or play sub-optimal arms somewhere. The proof makes rigorous the intuition that there is no continuous way for two players to always choose the top pair of actions.}
\label{fig:circle}
\end{figure}

\begin{proof}
First, we prove that if $\mcl{F}(P_i)$ is $(1,1), (2,2)$ or $(3,3)$ then we are immediately done.  Suppose without loss of generality that $\mcl{F}(P_i) = (1,1)$.  Then $\mcl{F}(P_{i-1})$ and $\mcl{F}(P_{i+1})$ cannot have either coordinate equal to $1$ or we would be immediately done. If $\mcl{F}(P_{i-1})$ has matching coordinates, then without loss of generality we may assume $\mcl{F}(P_{i-1}) = (2,2)$.  Then we must also have $\mcl{F}(P_{i+1}) = (3,3)$ since otherwise we would immediately be done.  But now considering all choices for $j,j'$ among $\{i-1, i , i + 1 \}$, we conclude that playing any two arms is within $0.01$ of optimal at one of these points.  However, by construction, the points $P_i$ are all sufficiently far from the line $x = y = z$ so this is impossible.  Now it remains to consider the case that $\mcl{F}(P_{i-1})$ does not have matching coordinates -- suppose without loss of generality that $\mcl{F}(P_{i-1}) = (2,3)$.  Then we must also have $\mcl{F}(P_{i+1}) = (2,3)$ or we would immediately be done.  Next, as before, considering all choices for $j,j'$ among $\{i-1, i , i + 1 \}$, we conclude that playing any two arms is within $0.01$ of optimal at one of these points which again is a contradiction.  

From now on we assume that for all $i$, $\mcl{F}(P_i)$ has distinct coordinates. Define the set
\[
X=\{(x,y,z)\in\R^3:x \geq \frac{y + z}{2} + 0.04\}\subseteq \R^3 \,.
\]
Next, we argue that for any $P_i \in X$, one coordinate of $\mcl{F}(P_i)$ must be $1$ or we are immediately done.  To see this, assume not and that without loss of generality $\mcl{F}(P_i) = (2,3)$.  Now consider $\mcl{F}(P_{i+1})$.  We cannot have both coordinates of $\mcl{F}(P_{i+1})$ be $1$ so without loss of generality the second coordinate is not $1$.  Then setting $(j,j')  = (i, i+1)$ gives \eqref{eq:loss} since by the definition of region $X$, arm $1$ is at least $0.04$ better than the worst arm in this region.

Next define $X'\subseteq \R^3$ by 
\[
X'=\{(x,y,z)\in\R^3:x \geq \frac{y + z}{2} + 0.05\} \,.
\]
Note that $X'\subseteq X$ so for all $P_i \in X'$, some coordinate of $\mcl{F}(P_i)$ must equal $1$.  We now argue that the coordinate equalling $1$ must actually be the same for all points in $X'$.  If this is not the case, then there must exist points $P_{i_1}, P_{i_2} \in X'$ such that $\mcl{F}(P_{i_1})$ and $\mcl{F}(P_{i_2})$ have different coordinates equal to $1$. Without loss of generality suppose $i_1 < i_2$.  By construction, the path $P_{i_1}, P_{i_1 + 1}, \dots , P_{i_2}$ cannot leave the region $X$, so there must exist two consecutive points, say $P_{i_3}, P_{i_3 +1}$, such that $\mcl{F}(P_{i_3})$ and $\mcl{F}(P_{i_3 +1})$ have different coordinates equal to $1$. Setting $(j,j') = (i_3, i_3 + 1)$ or $(j,j') = (i_3 + 1, i_3)$ immediately gives \eqref{eq:loss}.

Now we can define the regions $Y',Z'$ analogously to $X'$ (with the variables in the constraint permuted accordingly).  Repeating the above argument implies
\begin{enumerate}
    \item For all $P_i \in X'$, some fixed coordinate of $\mcl{F}(P_i)$ is $1$
    \item For all $P_i \in Y'$, some fixed coordinate of $\mcl{F}(P_i)$ is $2$
    \item For all $P_i \in Z'$, some fixed coordinate of $\mcl{F}(P_i)$ is $3$
\end{enumerate}
However, two of these ``fixed" coordinates must be the same.  Furthermore, by construction there must exist $P_i$  in $X' \cap Y'$, and similarly for $Y' \cap Z'$ and $Z' \cap X'$. This is a contradiction and concludes the proof.  
\end{proof}
Claim~\ref{claim:obstruction} motivates the following terminology.
\begin{definition}
The pair of points $(P, Q)\in (\R^3)^2$ is a $\gamma$-loss for the function $\mcl{F}:\R^3\to\{1,2, 3 \}^2$ if
\[
\mcl{V}(P) - \mcl{G}(\mcl{F}, P,Q) \geq \gamma  \,.
\]
\end{definition}

The main ingredient in the proof of Lemma~\ref{lem:regret-at-some-scale} will be the following result where we use Claim~\ref{claim:obstruction} to count the number of pairs of points  $P,Q \in \{\frac{0}{T}, \frac{1}{T},  \dots , \frac{T}{T} \}^3$ such that $P,Q$ are sufficiently close (roughly, $\norm{P - Q}_2 \sim 1/\sqrt{T}$) and $(P,Q)$ forms a $\Omega(1)$-loss.

\begin{claim}\label{claim:count}
Let $t$ be a positive integer.  Let $\mcl{B}$ denote the set of $(t+1)^3$ points $\{ \frac{0}{t}, \frac{1}{t},  \dots , \frac{t}{t}  \}^3 $.  For any function $\mcl{F}$ mapping $\R^3$ to ordered pairs among $\{1,2,3 \}$, there must be at least $\Omega(t^4)$ pairs of points $P,Q \in \mcl{B}$ such that 
\begin{itemize}
    \item All coordinates of $P$ and $Q$ are between $0.01$ and $0.99$
    \item $\norm{P - Q}_2 \leq 1/\sqrt{t}$
    \item $(P,Q)$ is a $0.01$-loss for $\mcl{F}$
\end{itemize}
\end{claim}
\begin{proof}
Let $\mcl{C}$ denote the set of integer multiples of $1/\sqrt{t}$ between $1$ and $2$.  Pick $C_1, C_2 \in \mcl{C}$.  Now consider the plane in $\R^3$ formed by $x + y + z = C_1$.  Within this plane, consider the circle of radius $0.1C_2$ centered around the point $(C_1/3,C_1/3, C_1/3)$.  Choose $\sqrt{t}$ evenly spaced points $Q_1, \dots , Q_{\sqrt{t}}$ on this circle.  Around each point $Q_i$, draw a ball of radius $0.01/\sqrt{t}$ (in $\R^3$).  

Consider a set of $\sqrt{t}$ points $P_1, \dots , P_{\sqrt{t}}$ obtained by picking exactly one point from each ball.  By Claim~\ref{claim:obstruction}, there must exist a pair $(P_i, P_{j})$ with $| i - j| \leq 2$ that is a $0.01$-loss for $\mcl{F}$.  Note that because $|i - j| \leq 2$, this pair must have $\norm{P_i - P_{j}}_2 \leq 1/\sqrt{t}$.  Now, we can use the above argument for \textit{any} choice of $P_1, \dots , P_{\sqrt{t}}$.  Note that the ball of radius $0.01/\sqrt{t}$ around each $Q_i$ contains $\Omega(t^{3/2})$ points of $\mcl{B}$.  Thus, there must be at least $\Omega(t^3)$ pairs of points within distance $1/\sqrt{t}$ that are a $0.01$-loss for $\mcl{F}$.

Now, we can aggregate over our choices of $C_1, C_2$.  Note that the sets of points considered for different choices of $C_1$ and $C_2$ are disjoint.  Also, clearly all points that we consider have all coordinates between $0.01$ and $0.99$.  Thus, overall there must be $\Omega(t^4)$ pairs of points $P,Q \in \mcl{B}$ satisfying the desired properties and we are done.
\end{proof}

We can now complete the proof of Lemma~\ref{lem:regret-at-some-scale} by applying Claim~\ref{claim:count} and aggregating over different timesteps with a counting argument.

\begin{proof}[Proof of Lemma~\ref{lem:regret-at-some-scale}]
Consider times $t = \{ T/2, T/2 + 1, \dots , T \}$.  We apply Claim~\ref{claim:count} for each such $t$; if the players have a randomized strategy, we apply Claim~\ref{claim:count} to each strategy in their joint distribution.  We then sum over $T/2\leq t\leq T$.  Overall, we obtain for each such $t$ a pair of points $P^t, Q^t\in\mathbb R^3$ and weight $w_{P^t,Q^t}\in [0,1]$ with the following properties:
\begin{itemize}
    \item The weight $w_{P^t,Q^t}$ is the probability that $(P^t, Q^t)$ is a $0.01$-loss for the strategy chosen by the players at timestep $t$
    \item All pairs $(P^t, Q^t)$ have all coordinates between $0.01$ and $0.99$ and $\norm{P - Q}_2 \leq 2/\sqrt{T}$
    \item $\sum_{t=T/2}^T w_{P^t, Q^t} = \Omega(T^5)$ 
\end{itemize}
For a set $\mcl{S} \subset \R^3$, we denote by 
\[
\sum_{\mcl{S}} w_{P^t, Q^t}=\sum_{\substack{t\in \mathbb Z\cap [T/2,T]\\\text{and } P^t,Q^t\in \mcl{S}}}w_{P^t,Q_t}
\]
the total weight from pairs $(P_t,Q_t)$ contained inside $\mcl{S}$.  Recall Definition~\ref{def:value-and-gap} and that $\Delta(\p)=p_2-p_3$ for $\p\in \mathbb R^3$.  For $1/\sqrt{T} \leq c \leq 1$, let $G_{c}$ denote the set of points $\p\in\R^3$ satisfying
\[
 \frac{c - 1/\sqrt{T}}{2} \leq \Delta(\p) \leq c 
\]
 There must exist $\Delta \geq 1/\sqrt{T}$ such that  $\sum_{G_\Delta} w_{P^t, Q^t} = \Omega(T^5/\log T)$.  Next cover $G_\Delta$ using $O(T^{3/2}\Delta )$ balls of radius $2/\sqrt{T}$ (recall $\Delta \geq 1/\sqrt{T}$).  Replacing each of these balls with a ball of radius $4/\sqrt{T}$ with the same center, each pair $(P^t,Q^t)$ appearing in the sum must be contained in such a ball.  In particular, \textbf{some} ball $B$ satisfies  
\[
    \sum_{B} w_{P^t, Q^t} =  \Omega(T^{7/2}/ (\Delta \log T)).
\]
We can increase the radius of $B$ to $10/\sqrt{T}$ and moving the center to some point $O$ such that $\Gap(O) \geq \Delta/2$ while ensure that this new ball $B'$ also satisfies $\sum_{B'} w_{P^t, Q^t} =  \Omega(T^{7/2}/ (\Delta \log T))$. 

To complete the proof, we compute the regret when the true instance $\p$ is the point $O$ identified above.  Observe that for all pairs $P^t, Q^t$ inside $B'$, the probability that one player observes $P^t$ is $\Omega(T^{-3/2})$, and the probability that the other player observes $Q^t$ is also $\Omega(T^{-3/2})$.  This follows from the multivariate local central limit theorem (see e.g. \cite{davis1995elementary}), or just the multinomial theorem, because $B'$ has radius $10/\sqrt{T}$ and all coordinates of $O$ are bounded away from $0$ and $1$. If both of these observation events occur, then the expected regret incurred at timestep $t$ is at least $0.009 w_{P^t, Q^t}$.  Indeed by definition, with probability $ w_{P^t, Q^t}$ the players play according to a strategy with $0.01$-loss at $(P^t, Q^t)$.  If the true instance were $P^t$, this would imply that the players incur $0.01$ expected regret.  Since the true instance $\p=O$ satisfies $\norm{\p - P^t}_2 \leq 10/\sqrt{T}$ the regret incurred is at least $0.009$.  Thus, as desired, the overall expected regret is at least 
\[
0.009 \sum_{B'} w_{P^t, Q^t} \cdot \Omega(T^{-3/2}) \cdot \Omega(T^{-3/2}) = \Omega\left(\frac{T^{7/2}}{\Delta \log T}\right) \Omega(T^{-3}) = \Omega\left( \frac{\sqrt{T}}{\Delta \log T} \right).
\]
\end{proof}

\subsection{Reduction to $2$ Players and $3$ Arms}

To complete the proof of Lemma~\ref{lem:regret-at-some-scale2} we reduce to the case of two players and three arms. The main idea is to simply add appropriate numbers of arms known to have reward $1$ or value $0$, however the details require some care.

\begin{proof}[Proof of Lemma~\ref{lem:regret-at-some-scale2}]
First consider a deterministic strategy $\mcl{A}$ for $m$ players and $K$ arms.  For any instance $\p \in \R^3$, let $\p_{\textsf{ext}}$ be the instance obtained by adding $m - 2$ arms with value $1$ and $K - m -1$ arms with value $0$.  Assume that the instance $\p$ has all coordinates $p_i\in[0.01,0.99]$.  Now consider running $\mcl{A}$ on this new instance.  We will prove that from $\cA$ we can construct a strategy $\mcl{A}'$ for $(K,m)=(3,2)$ such that for all $\p \in [0.01,0.99]^3$, we have
\begin{equation}\label{eq:reduce-strategy}
R_T(\cA' ; \bp) \leq O(R_T(\cA ; \p_{\textsf{ext}})) \,.
\end{equation}
Fix a timestep $t$ and instances $\p, \p_{\textsf{ext}}$.  For each arm $a$, let $q_a^1, \dots , q_a^m$ be the respective probabilities that each player plays this arm at time $t$ (note the only randomness is over their observations).  Let $S_1$ be the set of arms with value $1$ and $S_0$ be the set of arms with value $0$.  Then the expected regret incurred is at least  
\begin{align}\label{eq:regret-lower-bound}
\nonumber R = \Omega\Bigg( \sum_{a \in S_1} \left( (1 - q_a^1) \cdots ( 1 - q_a^m) + \left( (1 - q_a^1) \cdots ( 1 - q_a^m) + q_a^1 + \dots + q_a^m - 1 \right) \right) \\ + \sum_{a \in S_0} \left( q_a^1 + \dots +  q_a^m \right) \Bigg) \,.
\end{align}
Note that to obtain the above we used that the players' observations are independent and the fact that $p_{m+1}\in [0.01,0.99]$.  The first term above comes from the fact that $\Omega(1)$ regret is incurred if some arm of value $1$ is not played or if two players collide. The second term comes from the observation that $\Omega(1)$ regret is incurred whenever some player plays an arm of value $0$.

Now we describe our strategy $\mcl{A}'$ for a two player, three arm instance $\p$.  The two players play as follows.  At each timestep $t$, they augment their observations with $m - 2$ arms that always output $1$ and $K - m - 1$ arms that always output $0$ (this simulates an observation from $\p_{\textsf{ext}}$).  Now the first player applies the strategies of players $1,2, \dots , m$ in $\mcl{A}$ in that order on his augmented set of observations and plays according to the first strategy that dictates playing one of the original arms.  If no such strategy exists he plays arbitrarily.   The second player does the same thing except he examines the strategies of players $m,m-1, \dots , 1$ in $\mcl{A}$ in that order.

Now we compare the regret of $\mcl{A}'$ on $\p$ to the regret of $\mcl{A}$ on $\p_{\textsf{ext}}$.  It suffices to compare the regret at a fixed timestep $t$.  The key observation is that \eqref{eq:regret-lower-bound} implies that the following two properties must hold, or else the expected regret incurred at timestep $t$ by $\mcl{A}$ will be $\Omega(1)$.
\begin{enumerate}
    \item  For all $a \in S_1$, there exists $X\in[m]$ such that $q_a^X \geq  0.9$ and $ \sum_{Y \neq X} q_a^Y \leq 0.1$.
    \item For all $a \in S_2$, the inequality $\sum_{X} q_a^X \leq 0.1$ holds.
\end{enumerate}
In particular, there is one player ``responsible" for playing each of the arms in $S_1$.  Then there are two players left over who are responsible for playing the arms from $\p$.  Without loss of generality, we suppose that $S_1 = \{1,2, \dots , m -2 \}$ and that players $i_1,i_2, \dots , i_{m-2}$ are exactly the players responsible for playing these arms.  The remaining two players are labeled $i_{m-1}, i_m$.  The probability that the players in $\mcl{A}'$ do not play according to the players $i_{m-1}$ and $i_m$ is at most 
\[
\sum_{a = 1}^{m-2} (1 - q_a^{i_a}) + \sum_{a \in S_0 \cup S_1} (q_{a}^{i_{m-1}} + q_{a}^{i_m}) 
\]
because the only way this can happen is if one of the players $i_1, \dots , i_{m-2}$ deviates from the arm they are responsible for or one of $i_{m-1}$ or $i_m$ plays an arm in $S_0$ or $S_1$.  However, we assumed that $q_a^{i_a} \geq 0.9$ and $ \sum_{X \neq i_a} q_{a}^X \leq 0.1$ for any $a \in S_1$ so the above is at most $O(R)$ (for $R$ as in \eqref{eq:regret-lower-bound}).  Finally, if the players in $\mcl{A}'$ play according to players $i_{m-1}$ and $i_m$ then the regret incurred by $\mcl{A}'$ is at most the regret incurred by $\mcl{A}$. Combining cases proves \eqref{eq:reduce-strategy}.

Finally for randomized strategies $\mcl{A}$, we simply apply the above transformation to each possible combined strategy in the joint distribution. Combining \eqref{eq:reduce-strategy} with Lemma~\ref{lem:regret-at-some-scale} completes the proof. 
\end{proof}

%% file: alg-outline.tex
\section{High-Level Overview of the Algorithm}

For the remainder of the paper we focus on algorithmically achieving \eqref{eq:main-upper-bound}, and we begin with an overview. Our starting point is the collision-free algorithm of \cite{bubeck2021cooperative} achieving $O(K^{O(1)}\sqrt{T\log T})$ regret for any $\p$, which as we have seen is already a Pareto optimum. Their idea was to handle the inherent topological obstruction (as in Figure~\ref{fig:circle}) by inserting a thin ``skeleton'' region to partition the state space $[0,1]^K$ of arm estimates. This skeleton is shown in orange and purple in Figure~\ref{fig_partitionBB}. In their construction, the skeleton has width $\wt O(1/\sqrt{t})$ and is positioned \textbf{randomly}. They then define a piece-wise constant strategy on the partition regions. Each player computes his own empirical arm averages (a point in $[0,1]^K$) and then finds the region in the partition containing that point. He then plays using the label for that region, e.g. landing in a region labelled $(i,j)$ causes the first player to play arm $i$ and the second player to play arm $j$. \cite{bubeck2021cooperative} label the partition such that if players land in adjacent regions, they \textbf{never} collide. Moreover the only regret comes from skeleton regions, as all players choose the top two actions on the remainder of the partition. The skeleton's random location ensures low average regret for any $\p$.

Our new construction, shown in Figure~\ref{fig_partition_new}, adds three blue triangular regions into the diagram, which contain instances with a large gap. The idea is that each player first checks whether the gap of their empirical estimates is sufficiently large i.e. if it is clear what the top actions are. If it is, they land in the blue region, thus avoiding the complicated main part of the partition. The players suffer no regret after entering the blue region, which reduces the total regret for large gaps. The size of the blue triangular regions depends on the sequence $(\Delta_j)_{j\leq J}$. In short, the blue regions occupy a $1 - \Delta_j$ fraction of the area at times $t\in [t_j,t_{j+1})$, where $t_j$ is roughly equal to $\Delta_j^{-2}$. 

The downside of adding the blue triangles is that it is still necessary to surround them with skeleton regions to prevent collisions. We can still randomize the location of the boundary of the blue region, but it is much \textbf{less} random than in \cite{bubeck2021cooperative}. In particular, rather than having a uniformly random location in a constant size window, the window size now shrinks to roughly $\Delta_j$.  This magnifies the probability that an instance $\p$ with small gap falls into the skeleton by $\Delta_j^{-1}$. Altogether, a careful analysis of this trade-off yields the Pareto-optimal guarantee \eqref{eq:main-upper-bound}.

We remark that from the diagrams shown, it is natural to think that the blue regions may not be needed, and one could simply restrict the possible locations of the skeleton to avoid instances $\p$ with large gaps. This approach seems to suffice in the easier full-feedback problem.
\footnote{We believe regret $O\left(\frac{\log T}{\Delta^2}\right)$ might also be achievable by a simple arm elimination approach. Roughly, players cycle through arms in a fixed order until the top $m$ actions become clear.} 
However the bandit analysis leads to additional subtleties already present in \cite{bubeck2021cooperative} and amplified in our setting. The issue is that different players will not have reward estimates within $\wt O(1/\sqrt{t})$ of each other for suboptimal arms that have been eliminated from consideration long ago and thus may no longer land in neighboring regions. Therefore, significant care is required to ensure that all arms relevant for choosing the correct region are accurately estimated.  To do this, it is crucial for the partitions for different phases $t\in [t_{j-1},t_j)$ and $t\in [t_j,t_{j+1})$ to be compatible.  We have to be extremely careful about what changes we make to the skeleton between time-steps because if the partition changes, a player may end up near a decision boundary which requires accurate estimates of a completely different set of arms that the player has not explored.  A key property of our algorithm is that the blue regions only depend on the gap between the $m$-th and $m+1$-st arms and thus the estimation accuracy for sub-optimal arms that are eliminated early does not end up affecting the decision. 
As a result the blue regions have the special feature that it is ``safe'' for them to grow larger as the algorithm progresses. Our analysis heavily exploits this idea; see e.g. Lemma~\ref{lem:safe-with-margin}.

\begin{figure}[H]
    % \centering
        \centering
         \includegraphics[width=0.9\linewidth]{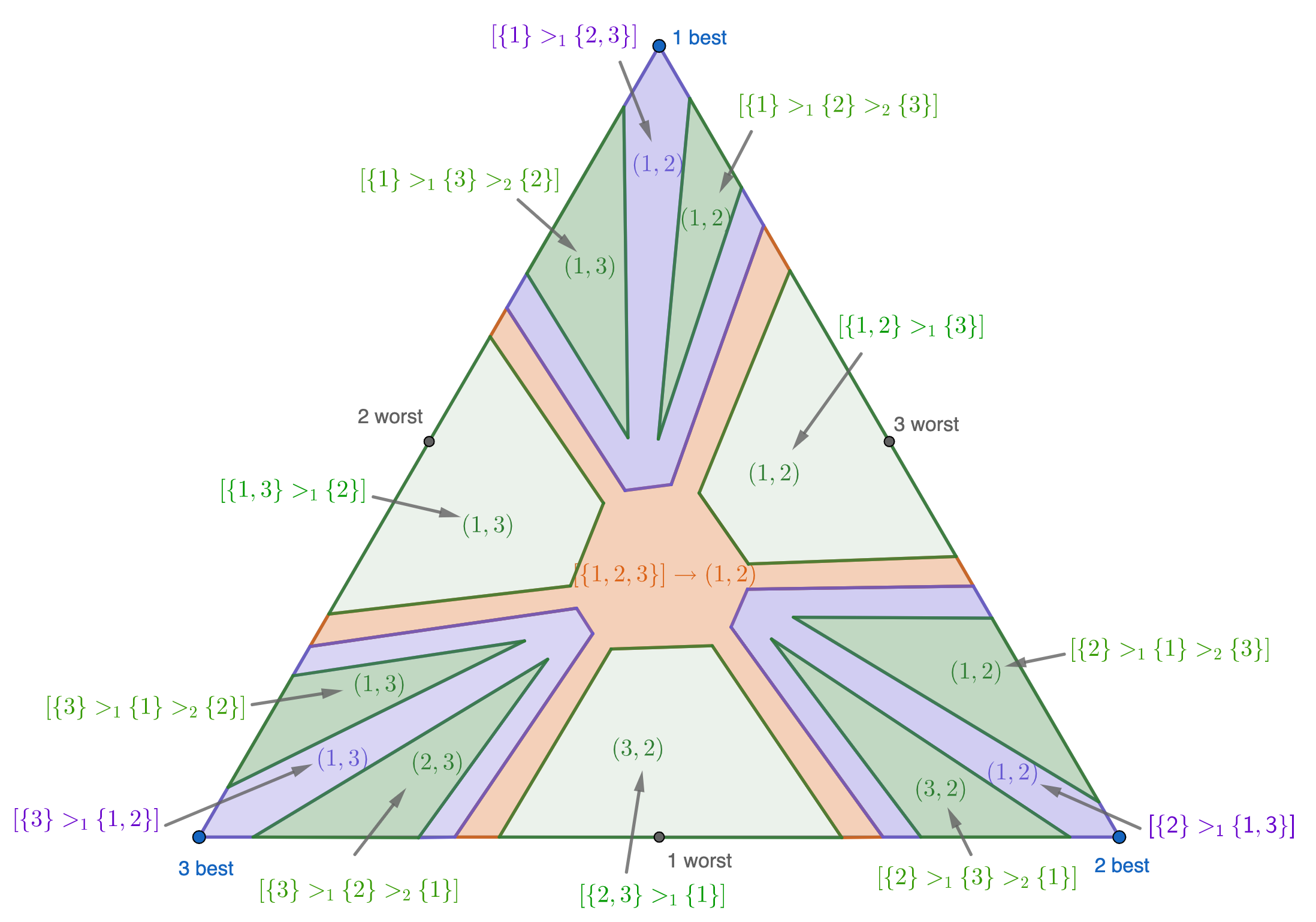}
       \caption{The partition of \cite{bubeck2021cooperative} for $2$ players and $3$ actions, restricted to a plane with $p_1+p_2+p_3$ constant. The parts are labelled by ordered pairs of actions, as well as vertices of the tree $\mathcal T_{K,m}$ (defined in the Appendix). The ``skeleton", shown in orange and purple, separates the green regions in which the top $2$ actions are played. Adjacent regions never result in a collision, leading to a collision-free algorithm. The random position of the skeleton ensures the average regret is $\widetilde O(K^{O(1)}\sqrt{T})$ for any $\p$.}
       \label{fig_partitionBB}
\end{figure}

\begin{figure}[H]
\centering
  \includegraphics[width=\linewidth]{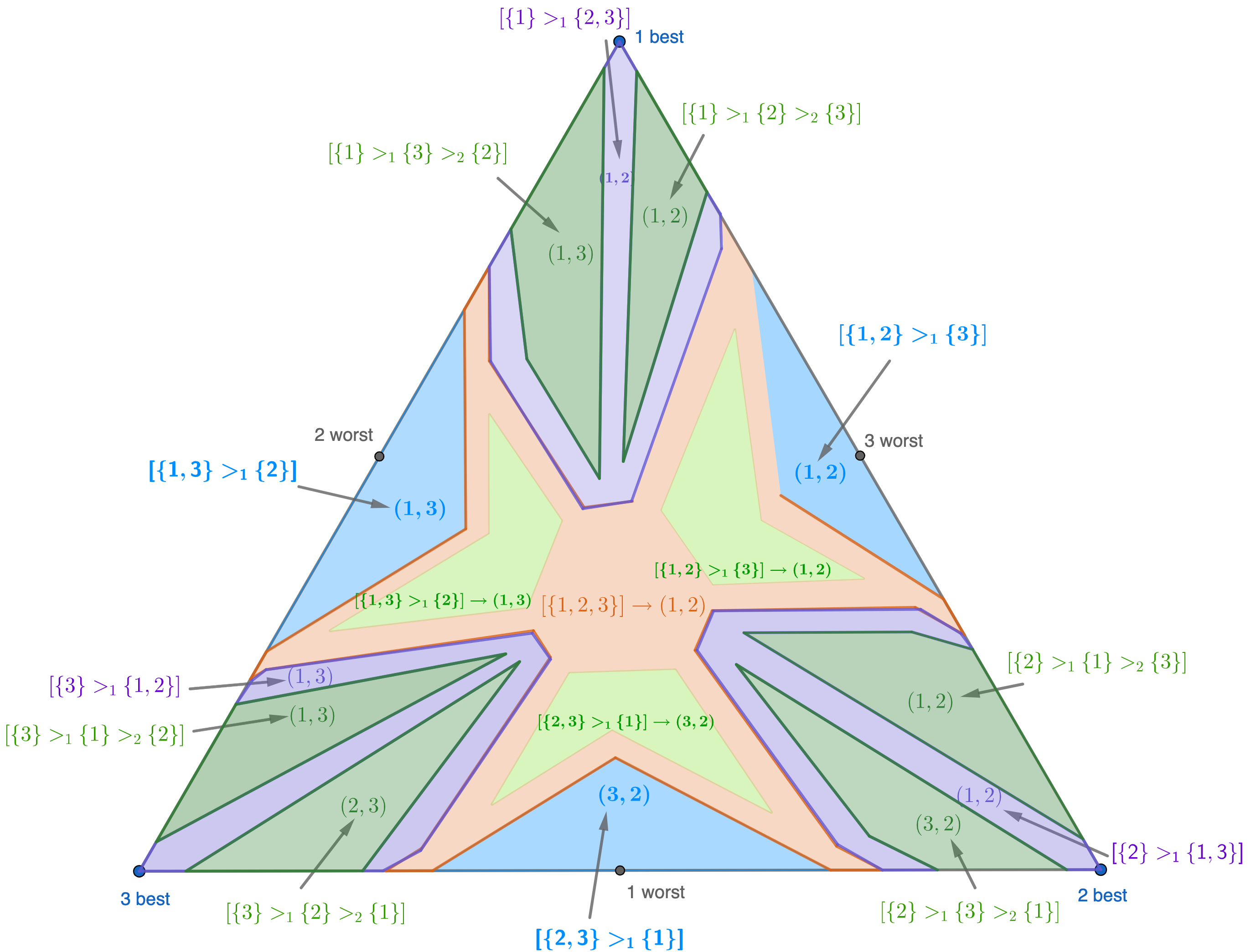}
  \caption{Our new partition, which adds three blue triangles of random size to the one above. These are defined ``first'' (in Line~\ref{line:m-cut} of Algorithm~\ref{alg:partition}) and cause perfect performance (zero regret) to be quickly reached when the gap $\Delta$ is large. The distance between the blue regions is approximately $\Delta_j$ for $t\in [t_j,t_{j+1})$, so the blue regions occupy most of the diagram for small $\Delta_j$. The tradeoff is that the orange and purple skeleton now occupies a more \textbf{predictable} region in the diagram, namely a thin neighborhood around the blue regions. This increases the regret for small values of $\Delta$, i.e. for points in the diagram close to a line segment connecting a vertex of the main triangle to its center. }
  \label{fig_partition_new}
\end{figure}

%% file: partition.tex
\section{Preliminaries for the Algorithm}\label{sec:alg-prelim}

In this section, we introduce the basic framework of our algorithm for proving Theorem~\ref{thm:main-upper}, following \cite{bubeck2021cooperative} with some new ingredients to enable instance dependence. Accordingly, we recall several definitions and lemmas therein.

\subsection{The Tree $\mathcal T_{K,m}$}

\label{subsec:Tkm}

We recall the tree-based partition of \cite{bubeck2021cooperative}.
First, an ordered set partition of $[K]$ has the form
\[
    P=\left[S_1> S_2 >\dots >S_j\right],
\]
where $(S_i)_{i=1}^j$ partition $[K]$. For example $\left[\{1,3,5\}>\{2,6,7\}>\{4\}\right]$ is an ordered set partition of the set  $\{1,2,3,4,5,6,7\}$ and is identical to $\left[\{5,1,3\}>\{6,2,7\}>\{4\}\right]$. We define a \emph{doubly ordered partition} (henceforth DOP) to be an ordered set partition in which the inequality signs are themselves ordered. Thus a DOP of $[K]$ has the form
\[
    P=\left[S_1>_{\sigma(1)} S_2 >_{\sigma(2)}\dots >_{\sigma(j-1)} S_j\right]
\]
for some permutation $\sigma\in \mathfrak{S}_{j-1}.$ For example \[
    \left[\{1,3,5\}>_1\{2,6,7\}>_2\{4\}\right] \quad \mbox{and} \quad \left[\{1,3,5\}>_2\{2,6,7\}>_1\{4\}\right]
\] 
are the two DOPs with underlying ordered partition $\left[\{1,3,5\}>\{2,6,7\}>\{4\}\right]$. The set of DOPs on $[K]$ are naturally arranged into a tree $\mathcal T_K$. More precisely, the root of $\mathcal T_K$ is the trivial DOP $\mathrm{ROOT} := \left[\{1,2,\dots,K\}\right]$ and, for every DOP
\[
    P_1=\left[S_1>_{\sigma(1)} S_2>_{\sigma(2)}\dots >_{\sigma(i-1)}S_i >_{j-1} S_{i+1}>_{\sigma(i+1)}\dots>_{\sigma(j-1)}S_j\right]
\]
different from $\ROOT$ (i.e. with $j \geq 2$), the parent of $P_1$ is
\[
    P=\left[S_1>_{\sigma(1)} S_2>_{\sigma(2)}\dots >_{\sigma(i-1)}S_i \cup S_{i+1}>_{\sigma(i+1)}\dots>_{\sigma(j-1)}S_j\right].
\]
In other words, descending the tree $\mathcal T_k$ amounts to adding inequalities $>_1, >_2, \dots$ in this order.

The important structure in multiplayer bandits turns out to be a subtree $\mathcal T_{K,m}\subseteq\mathcal T_K$, which allows us to focus on identifying the top $m$ actions without ``distracting'' inequalities. To define it, let $i(P)$ be the largest integer $i \geq 0$ such that $\sum_{j=1}^i |S_j| \leq m$ (for example $i(\mathrm{ROOT})  = 0$). We define the set 
\[
    A(P):=S_1 \cup \hdots \cup S_{i(P)}
\]
(with the convention $A(P) = \emptyset$ if $i(P)=0$), which corresponds to the set of actions that the DOP $P$ has already identified as being in the top $m$ actions. Next define the set $B(P)$ of actions that must be partitioned further to fully identify the top $m$ actions: if $|A(P)|=m$ then $B(P) := \emptyset$, and otherwise $B(P) := S_{i(P)+1}$. We denote $A_P=|A(P)|$ and $B_P=|B(P)|$. We now define $\mathcal T_{K,m}\subseteq \mathcal T_K$ as the subtree formed by paths from the root where only inequalities involving $B(P)$ may be added to a DOP $P$ at any time. In other words, we define $\mathcal{T}_{K,m}$ recursively as follows: let
\[
    P_1=\left[S_1>_{\sigma(1)} S_2>_{\sigma(2)}\dots >_{\sigma(i-1)}S_i >_{j-1} S_{i+1}>_{\sigma(i+1)}\dots>_{\sigma(j-1)}S_j\right]
\]
be a DOP and let
\[
    P=\left[S_1>_{\sigma(1)} S_2>_{\sigma(2)}\dots >_{\sigma(i-1)}S_i \cup S_{i+1}>_{\sigma(i+1)}\dots>_{\sigma(j-1)}S_j\right]
\]
be its parent. If $P \in \mathcal{T}_{K,m}$, then $P_1 \in \mathcal{T}_{K,m}$ if and only if $B(P)=S_i \cup S_{i+1}$. See Figure~\ref{fig:tree23} for an example. We also denote by $\L(\mathcal T_{K,m})$ the set of leaves of the tree $\mathcal{T}_{K,m}$. Note that the leaves of $\mathcal T_{K,m}$ are DOPs which determine the top $m$ actions. However not all DOPs determining the top $m$ actions are leaves of $\mathcal T_{K,m}$. For example we have \[[\left\{4,8\}>_2\{2,6,7\}>_1\{1,3,5\}\right]\in\mathcal T_{8,2} \quad \text{but} \quad \left[\{4,8\}>_1\{2,6,7\}>_2\{1,3,5\}\right]\notin\mathcal T_{8,2}.\] The latter holds because the parent DOP $\left[\{4,8\}>_1\{1,3,5,2,6,7\}\right]$ determines the top $2$ actions hence is already a leaf of $\mathcal T_{8,2}$.

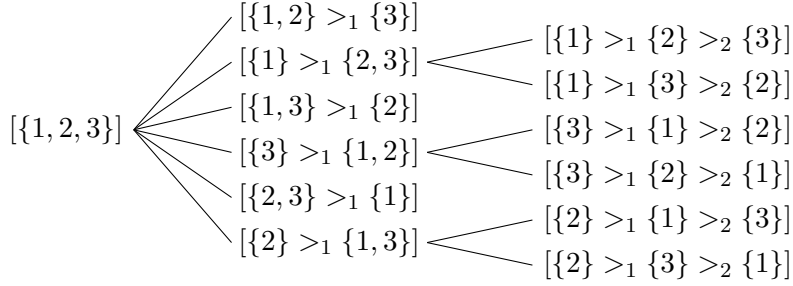
\begin{figure}
\begin{center}
\begin{tikzpicture}
\draw(0.9,0)--(2.2,1.5);
\draw(0.9,0)--(2.2,0.9);
\draw(0.9,0)--(2.2,0.3);
\draw(0.9,0)--(2.2,-0.3);
\draw(0.9,0)--(2.2,-0.9);
\draw(0.9,0)--(2.2,-1.5);

\draw(4.8,0.9)--(6.2,1.2);
\draw(4.8,0.9)--(6.2,0.6);
\draw(4.8,-0.3)--(6.2,0);
\draw(4.8,-0.3)--(6.2,-0.6);
\draw(4.8,-1.5)--(6.2,-1.2);
\draw(4.8,-1.5)--(6.2,-1.8);

\draw(0,0)node[txt]{$[\{1,2,3\}]$};

\draw(3.5,1.5)node[txt]{$[\{1,2\} >_1 \{3\}]$};
\draw(3.5,0.9)node[txt]{$[\{1\} >_1 \{2,3\}]$};
\draw(3.5,0.3)node[txt]{$[\{1,3\} >_1 \{2\}]$};
\draw(3.5,-0.3)node[txt]{$[\{3\} >_1 \{1,2\}]$};
\draw(3.5,-0.9)node[txt]{$[\{2,3\} >_1 \{1\}]$};
\draw(3.5,-1.5)node[txt]{$[\{2\} >_1 \{1,3\}]$};

\draw(8,1.2)node[txt]{$[\{1\} >_1 \{2\} >_2 \{3\}]$};
\draw(8,0.6)node[txt]{$[\{1\} >_1 \{3\} >_2 \{2\}]$};
\draw(8,0)node[txt]{$[\{3\} >_1 \{1\} >_2 \{2\}]$};
\draw(8,-0.6)node[txt]{$[\{3\} >_1 \{2\} >_2 \{1\}]$};
\draw(8,-1.2)node[txt]{$[\{2\} >_1 \{1\} >_2 \{3\}]$};
\draw(8,-1.8)node[txt]{$[\{2\} >_1 \{3\} >_2 \{1\}]$};
\end{tikzpicture}
\end{center}
\vspace{-1cm}
\caption{The tree $\mathcal{T}_{3,2}$, with $9$ leaves and $4$ inner nodes.}\label{fig:tree23}
\end{figure}

For $P\in\mathcal T_{K,m}$ we use $\mathrm{CHILDREN}(P)$ to denote its set of children, and $\mathrm{PARENT}(P)$ to denote its (unique) parent. For example we have
\[\mathrm{PARENT}\left(\left[\{1,3,5\}>_1\{2,6,7\}>_2\{4\}\right]\right)~=~\left[\{1,3,5\}>_1\{2,4,6,7\}\right].\] For convenience we will treat $\mathcal T_{K,m}$ as a partial order, so that $Q\preceq P$ means $Q$ is a ancestor of $P$. In particular, the root satisifes $\mathrm{ROOT}\preceq P$ for any $P\in\mathcal T_{K,m}$. Finally we denote by $d_{\mathcal T_{K,m}}$ the graph distance in the tree $\mathcal T_{K,m}$.

\begin{definition}
Let $\x \in [0,1]^K$ and $P\in\mathcal T_{K,m} \setminus \L(\mathcal T_{K,m})$. We define 
\[\range_{P}(\x):=\max_{k\in B(P)} x(k)- \min_{\ell \in B(P)} x(\ell).\] 
\end{definition}

\begin{definition}
Let $\x \in [0,1]^K$ and $P$ be a DOP of the form:
\begin{equation} \label{eq:DOPform}
P=\left[S_1>_{\sigma(1)} S_2>_{\sigma(2)}\dots >_{\sigma(i-1)}S_i >_{j-1} S_{i+1}>_{\sigma(i+1)}\dots>_{\sigma(j-1)}S_j\right].
\end{equation}
We define \[\gap_{P}(\x)=\min_{k \in S_{i}}x(k)-\max_{\ell \in S_{i+1}} x(\ell).\]
\end{definition}

In words, $\range_P(\x)$ represents the range of values in the set of coordinates for which the DOP $P$ has not yet identified whether they are in the top $m$ actions or not. On the other hand, $\gap_P(\x)$ represents how large was the ``cut" made by the DOP $P$ when we added its last inequality. The next easy lemma says that there always exists a ``large cut".

\begin{lemma}{\cite[Lemma 2.1]{bubeck2021cooperative}}
\label{prop:cover}
Let $\x \in [0,1]^K$ and $P\in\mathcal T_{K,m} \setminus \L(\mathcal T_{K,m})$. There exists a DOP $Q \in \mathrm{CHILDREN}(P)$ such that
\[\gap_{Q}(\x)\geq \frac1{K}\cdot \range_{P}(\x)\geq 0.\]
\end{lemma}

Compared to \cite{bubeck2021cooperative}, our algorithm differs in containing an initial step which tries to take advantage of a large gap $\Delta$ by finishing in one step.  This will allow us to design a strategy that does better on instances with large gaps. Accordingly, for $\x\in [0,1]^K$ with coordinates 
\[
    x(a_1)\geq x(a_2)\geq\cdots \geq x(a_K),
\]
we define the DOP $P_*(\x)$ by
\[
    P_*(\x)=[\{a_1,a_2,\cdots,a_m\}>_1 \{a_{m+1},\cdots,a_K\}].
\]

\subsection{Constructing the Partition}

We now construct the partition of $[0,1]^K$, which depends deterministically on a function $c:\mathcal T_{K,m}\to [0,\frac{1}{K}]$ as well as constants $\varepsilon, \delta>0$. The partition elements will be indexed by vertices of the tree $\mathcal{T}_{K,m}$, or in other words the partition is defined by a mapping $\mathcal P_{c, \varepsilon} : [0,1]^K \rightarrow \mathcal T_{K,m}$. This mapping is described algorithmically by Algorithm \ref{alg:partition}. 

Algorithm \ref{alg:partition} is similar to that of \cite{bubeck2021cooperative}.  The only difference in our algorithm is the addition of Line~\ref{line:m-cut} and Line~\ref{line:m-cut-padding} which check if the gap between the top $m$ values and the rest of the values is sufficiently large.

\begin{figure}[!h]

\SetKwFor{Loop}{loop}{}{end}

\begin{algorithm2e}[H]\label{alg:partition}
\caption{Definition of the mapping $\mathcal P_{c, \varepsilon} : [0,1]^K \rightarrow \mathcal T_{K,m}$.}

\SetAlgoLined\DontPrintSemicolon

\textbf{parameters:} $c:\mathcal T_{K,m}\to [0,\frac{1}{K}]$; $\varepsilon,\delta>0$. \textbf{input:} $x\in [0,1]^K$. \textbf{returns:} vertex $P \in \mathcal T_{K,m}$.

\tcp*[f]{Attempt to terminate immediately.} %Take $\delta=\frac{\Delta_j}{2}\geq poly(K,\log(T))\cdot\eps$ \alcomment{Can we just do $\delta \sim  [\eps_{t_j}, 2\eps_{t_j} ] $ ?}for $t\in [t_j,t_{j+1})$.}

\uIf{
    $\gap_{P_*(\x)}(\x)\geq \delta+\eps$ 
    \label{line:m-cut}
    }
    {
    \Return $P_*(\x)$
    }

\tcp*[f]{Now use the partition of \cite{bubeck2021cooperative}, except for Line~\ref{line:m-cut-padding}}

Initialize $P= \mathrm{ROOT}$.

\While{$P\notin \L(\mathcal T_{K,m})$}{\label{while}

\tcp*[f]{Padding layer for Line~\ref{line:m-cut}.}

\uIf{
    $\gap_{P_*(\x)}(\x)\geq \delta-(d_{\mathcal T_{K,m}}(P,\ROOT) + 2)\eps$ 
    \label{line:m-cut-padding}
    }
    {
    \Return $P$
    }

\For{$Q\preceq P$}{

Write $B(Q)=\{a_{A+1},\dots,a_{A+B}\}$, with $x(a_{A+1})\geq x(a_2)\geq\dots\geq x(a_{A+B})$ and $(A,B)=(|A(Q)|,|B(Q)|)$. \label{writeB(Q)} \\
\For{$j\in [B-1]$}{%
Define the child $Q_j$ of $Q$ by splitting $B(Q)$ into $\{a_{A+1},\dots,a_{A+j}\}>\{a_{A+j+1},\dots,a_{A+B}\}.$

\uIf{$\left|\gap_{Q_j}(\x)-c(Q)\cdot \range_Q(\x)\right| \leq (d_{\mathcal T_{K,m}}(P,Q) +1) \cdot 6\varepsilon$ \label{line:newskel}
}
{\Return $P$}
}
}
Write $B(P)=\{a_{A+1},\dots,a_{A+B}\}$, with $x(a_{A+1})\geq x(a_2)\geq\dots\geq x(a_{A+B})$ and $(A,B)=(|A(P)|,|B(P)|)$. \\
\For{$j\in [B-1]$}{\label{line:alg-order}
Define the child $P_j$ of $P$ by splitting $B(P)$ into $\{a_{A+1},\dots,a_{A+j}\}>\{a_{A+j+1},\dots,a_{A+B}\}.$

\uIf{$\gap_{P_j}(\x) \geq c(P) \cdot \range_P(\x)$ \label{line:childineq}
    }{

\tcp*[f]{By Lemma \ref{prop:cover} and $c(P) \leq \frac{1}{K}$, this occurs for at least one $j\in[B-1]$.}

$P\leftarrow P_j$\label{line:child} \\

\textbf{break} (go back to line~\ref{while})}

}%end for

\tcp*[f]{While loop terminated so $P$ is a leaf.
}}

\Return $P$

\end{algorithm2e}
\end{figure}

% Note that Algorithm~\ref{alg:partition} differs from that of \cite{bubeck2021cooperative} only in the ordering specified in line~\ref{line:alg-order}. Since the ordering within for loops could be taken arbitrarily in \cite{bubeck2021cooperative}, all properties proved therein continue to hold --- we recall some of them now.

% The first item will be useful to ensure the absence of collisions. The goal of the second is to state a ``consistency" property for different values of $\varepsilon$, which will be needed later in the bandit analysis.

\begin{lemma}{\cite[Lemma 2.2]{bubeck2021cooperative}}
\label{lem:topology}
We fix $\delta, \delta', \varepsilon>0$, $c : \mathcal{T}_{K,m} \to \left[ 0, \frac{1}{K} \right]$, and $\x, \y \in [0,1]^K$.
\begin{enumerate}
\item
If $||\x-\y||_{\ell^{\infty}}\leq\varepsilon/2$, then $d_{\mathcal T_{K,m}} \left( \mathcal P_{c, \varepsilon, \delta}(\x), \mathcal P_{c, \varepsilon, \delta}(\y) \right) \leq 1$.
\item
Let $P \in \mathcal{T}_{K,m}$ and assume that $|x(i)-y(i)| \leq \varepsilon/2$ for all $i \in A(P) \cup B(P)$. Let also $\varepsilon' \in (0,\varepsilon]$. Then it is not possible that $\mathcal P_{c, \varepsilon, \delta}(\x)$ and $\mathcal P_{c, \varepsilon', \delta' }(\y)$ are descendants of two distinct children of $P$.
\end{enumerate}
\end{lemma}
The proof of the above lemma is exactly the same for our new algorithm.  Roughly, this is because it is easy to verify that Line~\ref{line:m-cut} cannot break the lemma and Line~\ref{line:m-cut-padding} only causes the function $\mcl{P}$ to terminate higher in the tree which also cannot break the lemma.

\subsection{Coloring the Partition}\label{subsec:coloring}

To turn our partition into a full strategy we recall the existence of a rule specifying for each DOP $P$ which arm each player $X$ should play.  This part is the same as in \cite{bubeck2021cooperative}.

\begin{definition}
For a DOP $P$, define $\Feas_P\subseteq \binom{[K]}{m}$ to consist of all $m$-subsets of $[K]$ which comprise the top $m$ actions in some total ordering extending $P$.
\end{definition}
Note that this is more stringent than only requiring that each element might individually be in the top $m$. In particular, sequences in $\Feas_P$ contain all elements of $A(P)$ and a fixed size subset of $B(P)$.

\begin{definition}
An $m$-coloring of $\mathcal T_{K,m}$ is a function $F:\mathcal T_{K,m}\to [K]^m$. An $m$-coloring $F=(f_1, \hdots, f_m)$ is called \emph{collision-robust} if for any $P,Q\in\mathcal T_{K,m}$ with $d_{\mathcal T_{K,m}}(P,Q) \leq 1$ and any $i,j\in [m]$ with $f_i(P)=f_j(Q)$, one must necessarily have $i=j$. 
\end{definition}

\begin{lemma}{\cite[Lemma 2.4 and start of Section 4]{bubeck2021cooperative}}
\label{lem:color}
For any 
\[
    G:\mathcal T_{K,m}\to\binom{[K]}{m}
\] such that $G(P)\in \Feas_P$ for all $P$, there is a collision-robust $m$-coloring 
\[
    F:\mathcal T_{K,m}\to [K]^m,
\] 
where $F(P)$ is always a permutation of $G(P)$. Moreover (using shared randomness), one can arrange for each $F(P)$ to contain a uniformly random size $m-|A(P)|$ subset of $B(P)$ and to have uniformly random order conditioned on its set of elements.
\end{lemma}

%% file: analysis-main.tex
\section{Full Algorithm}\label{sec:full-alg}

The resulting bandit algorithm is also similar to \cite{bubeck2021cooperative}.  We use mappings of the form $\mathcal P_{c,\varepsilon_t,\delta_t}$, with a function $c$ chosen randomly at the beginning using the players' shared randomness.  To ensure sufficient exploration for each player, we will use a different, randomized coloring of $\mathcal{T}_{K,m}$ at each time-step.  Specifically, at each time $t$, we apply a uniformly random permutation $\pi_t:[K]\to [K]$ to the actions in defining the lexicographic ordering used in Section~\ref{subsec:coloring}, where the $\pi_t$ are independent. This defines a $\pi_t$-random coloring of the vertices of $\mathcal{T}_{K,m}$ and preserves the collision-robustness of Lemma~\ref{lem:color}. Moreover, by symmetry, the randomness of $\pi_t$ causes each $F(P)$ to contain a uniformly random subset of $B(P)$ of the appropriate size $m-|A(P)|$, and in particular to contain any arm $i\in B(P)$ with probability at least $\frac{1}{K}$.

We can now describe the strategy. Fix a sequence $\vD=(\Delta_0,\cdots,\Delta_J)$.  Note that without loss of generality, we may assume that
\[
    \Delta_j\geq 2\Delta_{j+1}
    \quad\quad
    \forall~ 0\leq j\leq J-1
\]
since we can simply modify the $\Delta_j$ to make this true and the guarantees in Theorem~\ref{thm:main-upper} will change by at most a constant factor.  For an integer $t$, define
\[
    \varepsilon_t =10000\sqrt{\frac{K^3\log(KT)}{t}}.
\]
For $j = 0,1, \dots J$, set $t_j = \lceil 10^{10} K^3 \log(KT) / \Delta_j^2 \rceil $.  For each $j$, choose $\delta_{t_j}$ uniformly at random from the interval $[\eps_{t_j}, 1.5 \eps_{t_j}]$ and set $\delta_t = \delta_{t_j}$ for $t_j\leq t< t_{j+1}$.

Let $n^X_t(i)$ be the number of times player $X\in [m]$ sampled arm $i\in [K]$ in the first $t-1$ time steps, and let $r^X_t(i)\leq n^X_t(i)$ be the amount of reward observed so far. Let $q^X_t(i)$ be the empirical estimate of $p(i)$ by player $X$ at the start of time $t$, defined by 
\[
    q^X_t(i)=\frac{r^X_t(i)}{n^X_t(i)}\in [0,1].
\]
For the first $T_0=10^9K\log(KT)$ time-steps, we simply have player $X$ play arm $X+t\pmod{K}$ at time $t$. After that, at time $t>T_0$ the players as before play via the mapping $\mathcal P_{c,\varepsilon_t,\delta_t}$, i.e. player $X$ plays arm $f_X \left( \mathcal{P}_{c,\varepsilon_t,\delta_t} \left( \mathbf{q}^X_t\right) \right)$, where $F=\left( f_1, \dots, f_m \right)$ is our $\pi_t$-random coloring, $c$ is a uniform function (independent of $t$) of the distance to the root i.e. 
\[
c(P) = c(d(P, \ROOT))
\]
where $c(0) ,c(1), \dots , c(K)$ are i.i.d. uniform on $[0, 1/K]$.  Note that $c$ is chosen once at the start of the algorithm and \textit{does not change} between time-steps.
%as in Section~\ref{subsec:coloring}.

%\alcomment{Need to blow this up more later lol}

\subsection{Properties of the Algorithm}

We now begin the analysis. The following definitions will be convenient. First, let $N_t^X(i)$ be the number of times $s\leq t-1$ satisfying
\begin{equation}
    \label{eq:in-relset}
    i\in A\lt(\mathcal{P}_{c,\varepsilon_s,\delta_s}(\q_s^X)\rt)
    \cup 
    B\lt(\mathcal{P}_{c,\varepsilon_s,\delta_s}(\q_s^X)\rt).
\end{equation}

\begin{definition}
The arm $i\in [K]$ is \textbf{well-explored} at time $t$ if $N_t^X(i)=t-1$ for all players $X$. If \eqref{eq:in-relset} does not hold for some $(i,X,s)$ we say arm $i$ was \textbf{rejected} at time $s$.
\end{definition}

We define the events:
\begin{align*}
\Omega_1
&=
\left\{ \forall t\geq T_0,i\in [K], X\in [m], 
\text{ we have } 
q_t^X(i)-p(i)|<\frac{\varepsilon_{n_t^X(i)}}{100K^{3/2}} \right\},
\\
\Omega_2
&=
\Big\{ 
    \forall t\geq T_0, i\in [K],
    n_t^X(i)\geq \left\lfloor \frac{N_t^X(i)}{2K} \right\rfloor 
\Big\}
% \Omega_2
% &=
% \Big\{ 
%     \forall t\geq T_0, i\in [K],
%     \text{ if } i \text{ is well-explored at time } t,
%     \text{ then }
%     n_t^X(i)\geq \left\lfloor \frac{t}{2K} \right\rfloor 
% \Big\}
,\\
\Omega_{~}&=\{\Omega_1\text{ and } \Omega_2\}.
\end{align*}

We will often use the following event which is a trivial consequence of $\Omega_2$:
\[
\Omega_2'
=
\Big\{ 
    \forall t\geq T_0, i\in [K],
    \text{ if } i \text{ is well-explored at time } t,
    \text{ then }
    n_t^X(i)\geq \left\lfloor \frac{t}{2K} \right\rfloor 
\Big\}.
\]
Observe that 
\[
    \frac{\varepsilon_{t/(2K)}}{100K^{3/2}}\leq \frac{\varepsilon_t}{10K}.
\] 
Hence if $\Omega$ holds, then
\[
    |q_t^X(i)-p(i)|<\frac{\varepsilon_t}{10K}
\]
for all players $X$ if $i$ is well-explored at time $t$.  We first note that the event $\Omega$ holds with high probability.  The proof follows from standard concentration inequalities.
\begin{lemma}
\label{lem:Omega}
Using the above strategy, for any choice of $\left( C(h) \right)$ and any $\p\in [0,1]^K$, we have:
\[\mathbb P[\Omega]\geq 1-\frac{1}{T}.\]
\end{lemma}

\begin{proof}
We show that each of $\Omega_1,\Omega_2$ have probability at least $1-\frac{1}{2T}$. For $\Omega_1$ this follows immediately from Hoeffding's inequality.

We now show $\mathbb P[\Omega_2]\geq 1-\frac{1}{2T}$. This is where we will use the randomization of the coloring using the $\pi_t$ to explore evenly. For $N_t^X(i)$ small, we use the initial sampling phase. Indeed, for $N_t^X(i)\leq \frac{3T_0}{2}$ the inequality $n_t^X(i)\geq \frac{N_t^X(i)}{2K}$ is immediate given our initial $T_0$ rounds of perfectly uniform sampling. Now, assume $N_t^X(i)\geq \frac{3T_0}{2}$ and fix $X\in [m]$, and $i\in [K]$. 
% Let $E_{t,X}(i)$ denote the event that there is $P$ for which $i \in A(P) \cup B(P)$ and, for all $s \leq t$, we have $\mathcal{P}_{c,\varepsilon_t,\delta_t}(\q_s^X) \preceq P$. 
Since we use a uniform random permutation $\pi_t$ at each time $s \leq t$, conditionally on everything that happened before, the probability for $X$ to play $i$ is at least $\frac{1}{K}$ as long as 
\[i\in A\lt(\mathcal{P}_{c,\varepsilon_s,\delta_s}(\q_s^X)\rt)
    \cup 
    B\lt(\mathcal{P}_{c,\varepsilon_s,\delta_s}(\q_s^X)\rt).
\]
It follows that
\[
    \mathbb P\left[ n_t^X(i)<\frac{N_t^X(i)}{2K}\right]
    \leq 
    \mathbb P\left[ Bin\left(N_t^X(i)-T_0,\frac{1}{K}\right)\leq \frac{N_t^X(i)-T_0}{2K}\right].
\]
As $N_t^X(i)-T_0\geq 1000K\log(KT)$, the right hand probability is at most $\frac{1}{2mKT^2}$ by applying the multiplicative Chernoff estimate $\mathbb P \left( Bin(N,p)\leq \frac{Np}{2} \right)	\leq e^{-\frac{Np}{8}}$ in \cite[Theorem 4.5]{mitzenmacher2017probability}. Union bounding over all $X,i$ and all times concludes the proof.
\end{proof}

In light of Lemma~\ref{lem:Omega}, we can condition on event $\Omega$ holding.  We will then prove several deterministic properties about the behavior of our algorithm. As time increases and we get further in the tree $\mathcal{T}_{K,m}$, we need to have a reasonable estimate of $\p$ to know which arms to keep exploring, but we also need to explore the right arms to maintain good estimates (and avoid collisions). We begin with some results that help to ensure a smooth transition to the phase where Line~\ref{line:m-cut} of Algorithm~\ref{alg:partition} applies. 

\begin{lemma}\label{lem:safe-with-margin}
Assume $\Omega$ holds.  Then for all $t$, at time $t$ in the execution of the algorithm, all arms that are within $\eps_t$ of the top $m$ are well-explored.
\end{lemma}
\begin{proof}
Assume for the sake of contradiction that the desired statement fails for the first time at time $t$.  There must be some arm $k$ within $\eps_t$ of the top $m$ that is rejected.  

Note that by the definition of $\Omega$ and the fact that $t$ is minimal, we must have the following properties for all $X \in [m]$
\begin{enumerate}
    \item $|\q_t^X(k) - \p(k)| < \frac{\eps_t}{10K} \,. $
    \item For all of the arms $i$ that are among the top $m$
    \[
    |\q_t^X(i) - \p(i)| < \frac{\eps_t}{10K} \,.
    \]
    \item For an arm $j$ that has gap $\eps \geq \eps_t$ to the top $m$, for all players $X$, 
    \[
    |\q_t^X(j) - \p(j)| < \frac{\eps}{9K} \,.
    \]
\end{enumerate}
Note the last property is because we find the largest $t'$ such that $\eps_{t'} \geq \eps$ and then use that the statement of the lemma holds up to time $t'$.

Now we consider when the arm $a_k$ is rejected by some player $X$.  Note there is some set of at least $m$ arms $\{i_1, \dots , i_m \}$ that are not rejected by player $X$.  Then we must have
\[
\q_t^X(i_c)  \geq \q_t^X(k) + 2\eps_t
\]
for all $c = 1,2, \dots , m$.  However, combining the three properties listed above immediately gives a contradiction.  This completes the proof of the lemma.
\end{proof}

\begin{lemma}\label{lem:stay-at-special}
Suppose that at time $s$, some player $X$ has
\[
\mathcal{P}_{c, \varepsilon_s, \delta_s}(\q_s^X ) = P_*(\q_s^X) \,.
\]
Then $P_*(\q_s^X) = P_*(\p)$ and for all players $Y$ and all timesteps $t \geq s$, we have 
\[
    \mathcal{P}_{c,\varepsilon_t,\delta_t}(\q_t^Y)
    \in 
    \{\ROOT, P_*(\p)\}.
\]
Furthermore, for all timesteps $t \geq 10s$, we have
\[
\mathcal{P}_{c,\varepsilon_t,\delta_t}(\q_t^Y) = P_*(\p) \,.
\]
\end{lemma}

\begin{proof}
We first prove the first statement that $\mathcal{P}_{c,\varepsilon_t,\delta_t}(\q_t^Y) \in \{\ROOT, P_*(\p)\}$.  In order to have $\mathcal{P}_{c, \varepsilon_s, \delta_s}(\q_s^X ) = P_*(\q_s^X)$, player $X$ must reject all but the top $m$ arms since by Lemma \ref{lem:safe-with-margin}, the top $m$ arms can never be rejected.  This immediately implies $P_*(\q_s^X) = P_*(\p)$.  Now without loss of generality let $s$ be minimal.  Let $a$ be an arm in the top $m$ and let $a'$ be an arm not in the top $m$.  We must have
\begin{equation}\label{eq:special-gap}
    \q^X_{s}(a)\geq \q^X_{s}(a')+\delta_{s}+\eps_{s}.
\end{equation}
If the arm $a'$ is never rejected up to timestep $s$, then by the assumption that $\Omega$ holds, we have for all players $Y$ and all timesteps $t \geq s$,
\begin{equation}\label{eq:desired}
\q^Y_{t}(a) \geq \q^Y_t(a') + \delta_s \geq \q^Y_t(a') + \delta_t \,.
\end{equation}
Next, consider if some player $Z$ rejects arm $a'$ at some timestep $s' < s$ where we choose $s'$ minimal.  By the minimality of $s$, this must occur in the main algorithm execution.  Since by Lemma~\ref{lem:safe-with-margin}, the top $m$ arms are never rejected we must have
\[
\q^Z_{s'}(a)\geq \q^Z_{s'}(a')+ 6\eps_{s'} \,.
\]
Let $s \in [t_j , t_{j+1})$.  Note that if $s' \leq t_j$ then $5\eps_{s'} \geq \delta_{t_j} = \delta_{s}$.  In this case, combining with the assumption that $\Omega$ holds, we again get that \eqref{eq:desired} holds.

Now consider if $s' \geq t_j$.  Note that then we must also have 
\[
\q^Z_{s'}(a) \leq \q^z_{s'}(a') + \delta_{s'} - \eps_{s'} = \q^Z_{s'}(a') + \delta_{s} - \eps_{s'} 
\]
since otherwise Algorithm~\ref{alg:partition} will return the $\ROOT$ in Line~\ref{line:m-cut-padding}.  However, combining the above with the assumption that event $\Omega$ holds immediately contradicts \eqref{eq:special-gap}.  Thus, we have actually shown that \eqref{eq:desired} always holds.  Since the choice of arms ($a$ from among the top $m$ and $a'$ from not among the top $m$) was arbitrary, we deduce that for all players $Y$ and all timesteps $t \geq s$
\[
\mathcal{P}_{c,\varepsilon_t,\delta_t}(\q_t^Y)
    \in 
    \{\ROOT, P_*(\p)\} \,.
\]
The second part follows almost immediately from the same argument.  We still have \eqref{eq:special-gap} for any arm $a$ in the top $m$ and arm $a'$ not in the top $m$.  Then when $t \geq (10K^2)s$, we can strengthen \eqref{eq:special-gap} to
\[
\q^Y_{t}(a) \geq \q^Y_t(a') + \delta_s + \eps_s/2 \geq \q^Y_t(a') + \delta_t + \eps_t \,.
\]
The argument in the other case can be directly modified as well.
\end{proof}

\subsection{Key Characterization}

Now we prove the key characterization about the behavior of the algorithm, stated below.

\begin{lemma}\label{lem:induction-step}
If event $\Omega$ holds, then there exists a path $P_*, P^0, \dots , P^k$ in the tree $\mathcal{T}_{K,m}$ such that $P_* = P_*(\p), P^0 = \ROOT$ and $P^0 \prec P^1\prec \dots \prec P^{k_t}$ and the following property holds: for all timesteps $t$ and all players $X,Y$, the vertices $\mathcal{P}_{c, \eps_t, \delta_t}(\q_t^X)$ and $\mathcal{P}_{c, \eps_t, \delta_t}(\q_t^Y)$ are adjacent and on the path.
\end{lemma}
\begin{proof}
First by Lemma~\ref{lem:stay-at-special}, if Line~\ref{line:m-cut} ever triggers, then all players will stay at $\{\ROOT, P_* \}$ for all future timesteps so we are done.

Now assume for the sake of contradiction that the hypothesis first fails at timestep $t$. We proceed in two cases, depending on how the failure occurs. Below, for each $j\geq 1$ we define $P^j$ a priori as follows. (This definition will eventually coincide with the statement of this lemma.) Let $s=s_j$ be the first time at which $\cP_{c,\eps_s,\delta_s}(\q_s^X)\in\cT_{K,m}$ has depth at least $j$ for some player $X$, and let $P^j$ be the ancestor of $\cP_{c,\eps_s,\delta_s}(\q_s^X)$ at depth $j$. In other words $P^j$ is the the first depth $j$ vertex for which 
\[
    \cP_{c,\eps_s,\delta_s}(\q_s^X) \succeq P^j
\]
ever holds for some player $X$. We break ties arbitrarily if needed. 

\paragraph{Case 1: A Player Strays From the Path}

In this case we assume for sake of contradiction that there exist players $X,Y$ and $s \leq t$ such that $\mathcal{P}_{c, \eps_s, \delta_s}(\q_s^X)$ and $\mathcal{P}_{c, \eps_t, \delta_t}(\q_t^Y)$ are not on the same path from the root, i.e. neither is an ancestor of the other.  We have already assumed that $t$ is minimal for this to hold, and we now also assume that $s$ is minimal given $t$. Let $P^j$ be the lowest common ancestor of  $\mathcal{P}_{c, \eps_s, \delta_s}(\q_s^X)$ and $\mathcal{P}_{c, \eps_t, \delta_t}(\q_t^Y)$.  By the minimality of $s$, we know that all arms in $A(P^j)\cup B(P^j)$ are well-explored up to time $s$.  Thus, by the assumption $\Omega$, for any $i \in A(P^j) \cup B(P^j)$, we have
\[
|q_s^X(i) - p(i)| \leq \frac{\eps_s}{10K} , |q_t^Y(i) - p(i)| \leq \frac{\eps_s}{10K} \,.
\]
Now by Lemma~\ref{lem:topology}, this contradicts the fact that $\mathcal{P}_{c, \eps_s, \delta_s}(\q_s^X)$ and $\mathcal{P}_{c, \eps_t, \delta_t}(\q_t^Y)$ are descendants of different children of $P^j$.  Thus actually our initial assumption was false, i.e. the hypothesis cannot first fail from a player straying from the path.

\paragraph{Case 2: Two Players are Not Adjacent}

We now consider the case that the hypothesis first fails at time $t$ because $\mathcal{P}_{c, \eps_t, \delta_t}(\q_t^X)$ and $\mathcal{P}_{c, \eps_t, \delta_t}(\q_t^Y)$ are not adjacent.  Hence we assume for sake of contradiction that $\mathcal{P}_{c, \eps_t, \delta_t}(\q_t^X) = P^{j_1}$ and $\mathcal{P}_{c, \eps_t, \delta_t}(\q_t^Y)    = P^{j_2}$ where without loss of generality $j_2 > j_1 + 1$.  $P^{j_1}$ cannot be a leaf, and so Algorithm~\ref{alg:partition} must terminate before reaching a leaf for player $X$.  This termination occurs either because of Line~\ref{line:newskel} or Line~\ref{line:m-cut-padding} so we break into two subcases.

\paragraph{Case 2.1: Line~\ref{line:newskel} is the Cause}

We first consider the case that Algorithm~\ref{alg:partition} terminates due to Line~\ref{line:newskel}.  Then as $P^{j_1}, P^{j_2}$ are not adjacent, we must have the following inequalities for some $j \leq j_1$ and some child $Q$ of $P^j$:
\begin{equation}\label{eq:diff-upperbound}
|\gap_{Q}(\q_t^X)  - c(P^j) \cdot \range_{P^j}(\q_t^X)| \leq 6K\eps_t; 
\end{equation}
\begin{equation}\label{eq:gap-bound}
|\gap_{Q}(\q_t^Y)  - c(P^j) \cdot \range_{P^j}(\q_t^Y)| \geq  |\gap_{Q}(\q_t^X)  - c(P^j) \cdot \range_{P^j}(\q_t^X)| +  6\eps_t \,.
\end{equation}
Let $s$ be the first time that some player $Z$ (possibly equal to $X$ or $Y$) satisfies $\mathcal{P}_{c, \eps_s, \delta_s}( \q_s^Z) = P^{j'}$ for some $j' > j$.  Note that obviously $s \leq t$.  Note that we must have
\begin{equation}\label{eq:diff-lowerbound}
|\gap_{Q}(\q_s^Z)  - c(P^j) \cdot \range_{P^j}(\q_s^Z)| \geq 6 \eps_s \,. 
\end{equation}
By the minimality of $s$ and the assumption that $\Omega$ holds, we have that for all $i \in A(P^j) \cup B(P^j)$ and all timesteps $t' \geq s$,
\begin{equation}\label{eq:estimate-bounds}
|q_{t'}^X(i) - p(i) | , |q_{t'}^Y(i) - p(i) |, |q_{t'}^Z(i) - p(i) | \leq \frac{\eps_s}{10K} \,. 
\end{equation}
Combining Equation~\eqref{eq:estimate-bounds} with \eqref{eq:diff-lowerbound} and \eqref{eq:diff-upperbound}, we find
\begin{align*}
    6\eps_s
    &\leq 
    |\gap_{Q}(\q_s^Z)  - c(P^j) \cdot \range_{P^j}(\q_s^Z)|
    \\
    &\leq
    |\gap_{Q}(\q_t^X)  - c(P^j) \cdot \range_{P^j}(\q_t^X)| + 2\max_{i\in A(P^j)\cup B(P^j)}|\q_t^X(i)-\q_s^Z(i)|
    \\
    &\leq 
    6K\eps_t+\eps_s.
\end{align*}
Rearranging implies
\[
\eps_s \leq 2K\eps_t  \,.
\]
However \eqref{eq:estimate-bounds} also implies
\[
    |\gap_{Q}(\q_t^Y)  - c(P^j) \cdot \range_{P^j}(\q_t^Y)| -  |\gap_{Q}(\q_t^X)  - c(P^j) \cdot \range_{P^j}(\q_t^X)|\leq \frac{\eps_s}{K}\leq 2\eps_t
\]
which contradicts \eqref{eq:gap-bound}.
% However, plugging the above into \eqref{eq:estimate-bounds} and then combining with \eqref{eq:gap-bound} yields a contradiction.  
Thus actually, returning $P^{j_1}$ when running Algorithm~\ref{alg:partition} to compute $\mathcal{P}_{c, \eps_t, \delta_t}(\q_t^X)$ cannot occur at a result of Line~\ref{line:newskel}.  

\paragraph{Case 2.2: Line~\ref{line:m-cut-padding} is the Cause}

It remains to consider the case that returning $P^{j_1}$ occurs due to Line~\ref{line:m-cut-padding}.  Then we must have
\begin{equation}\label{eq:diff-upperbound2}
\delta_t - (j_1 + 2)\eps_t \leq \gap_{P_*(\q_t^X)}(\q_t^X) \leq \delta_t + \eps_t
\end{equation}
\begin{equation}\label{eq:gap-bound2}
\gap_{P_*(\q_t^Y)}(\q_t^Y) \leq \delta_t - (j_1 + 3)\eps_t   \,.
\end{equation}
Let $t \in [t_j, t_{j+1})$ for some $j$. 
% \mscomment{I guess $j$ is a notation clash}
If an arm $a'$ is ever rejected at time $s_0 < t_j$ by some player $Z$, where we choose $s_0$ to be minimal, then we must have 
\[
\q_{s_0}^Z(a) - \q_{s_0}^Z(a') \geq 6\eps_{s_0}
\]
for all arms $a$ in the top $m$ (we ignore the case that Line~\ref{line:m-cut} executes, as if this ever occurs we are immediately done by Lemma~\ref{lem:stay-at-special}).  Note that $2\eps_{s_0} \geq \delta_{t_j}$.  Thus, by the assumption that $\Omega$ holds, we have that for any timestep $s' \geq {s_0}$ and any player $Z'$, for all arms $a$ in the top $m$, 
\begin{equation}\label{eq:early-elimination}
\q_{s'}^{Z'}(a) - \q_{s'}^{Z'}(a') \geq 2.5\delta_{t_j} \,.
\end{equation}
Next consider the first time $s \in [ t_j , t]$ such that some player $Z$ satisfies $\mathcal{P}_{c, \eps_s, \delta_s}(\q_s^Z) \neq \ROOT$.  Then 
\begin{equation}\label{eq:comparison}
\gap_{P_*(\q_s^Z) }(\q_s^Z ) \leq \delta_s -  \eps_s = \delta_{t_j} - \eps_s \,.
\end{equation}
By \eqref{eq:early-elimination}, any arm that is ever rejected before time $t_j$ is not relevant for computing the $\gap$ above.  Thus, using the assumption that $\Omega$ holds,
\begin{equation}\label{eq:gap-stays-same}
|\gap_{P_*(\q_t^X)}(\q_t^X) - \gap_{P_*(\q_s^Z) }(\q_s^Z )|, |\gap_{P_*(\q_t^Y)}(\q_t^Y) - \gap_{P_*(\q_s^Z) }(\q_s^Z )| \leq 0.4 \eps_s/K \,.
\end{equation}
Combining \eqref{eq:diff-upperbound2}, \eqref{eq:comparison} and \eqref{eq:gap-stays-same} we have
\[
\eps_s \leq 2(j_1 + 1)\eps_t \leq 2K\eps_t \,.
\]
Plugging this back into \eqref{eq:gap-stays-same} and using the triangle inequality, we deduce
\[
|\gap_{P_*(\q_t^X)}(\q_t^X) - \gap_{P_*(\q_t^Y)}(\q_t^Y) | \leq 0.8\eps_t \,.
\]
However, subtracting \eqref{eq:diff-upperbound2} and \eqref{eq:gap-bound2} immediately gives a contradiction.  This completes Case 2.2 and hence the proof.
\end{proof}

%% file: regret.tex
\subsection{Regret Analysis}

Now we can analyze the regret of our full algorithm.  Note that the source of regret is exactly when some player stops before reaching a leaf when running Algorithm~\ref{alg:partition}.  In Lemma~\ref{lem:special-stop} and Lemma~\ref{lem:cut-stop}, we upper bound the probability that this happens.

\begin{lemma}\label{lem:special-stop}
For any instance $\p$, timestep $t \in [t_j, t_{j+1})$ and player $X$, we have
\[
\Pr[ \delta_{t} - (K + 2) \eps_t \leq \gap_{P_*(\q_t^X) }( \q_t^X) \leq \delta_t + \eps_t ] \leq \frac{5 \eps_{t/(20K^2)}}{\eps_{t_j}} \,.
\]
where the randomness is over the observations and the random choices of the algorithm.
\end{lemma}
\begin{proof}
By Lemma~\ref{lem:Omega} and the observation that $\frac{\eps_{t/(20K^2)}}{\eps_{t_j}}\geq \frac{1}{T}$ it suffices to show that
\[
    \Pr[ \delta_{t} - (K + 2) \eps_t \leq \gap_{P_*(\q_t^X) }( \q_t^X) \leq \delta_t + \eps_t \text{ and } \Omega ] \leq \frac{4 \eps_{t/(20K^2)}}{\eps_{t_j}}.
\]
We thus assume $\Omega$ holds for the remainder of the proof and let $E$ be the event that
\[
\delta_{t} - (K + 2) \eps_t \leq \gap_{P_*(\q_t^X) }( \q_t^X) \leq \delta_t + \eps_t.
\]
If $t \leq 20K^2 t_j$ then the right-hand side is large than $1$ so the claim is trivially true.  Also, if Line~\ref{line:m-cut} ever triggers before timestep $t/(10K^2)$, then the event $E$ cannot happen by Lemma~\ref{lem:stay-at-special}, so we assume that Line~\ref{line:m-cut} does not trigger before then.  Note that for timesteps $s \in [t_j, 1.5t_j]$, all players must play $\ROOT$ because $\delta \leq 1.5 \eps_{t_j} \leq 2\eps_s$ so Line~\ref{line:m-cut-padding} always triggers immediately. 

Now let $s$ be the first time in the interval $[t_j, t_{j+1})$ for which $\mcl{P}_{c, \eps_s, \delta_s}(\q_s^X) \neq \ROOT$ holds for any player $X$.  Then we must have either 
\[
\gap_{P_*(\q_{s}^X) } (\q_{s}^X)\leq \delta_{s} - 2\eps_{s}
\quad\quad
\text{ or } 
\quad\quad
\gap_{P_*(\q_{s}^X) } (\q_{s}^X) \geq \delta_{s} + \eps_{s} \,.
\]
As we assume event $\Omega$ holds, it follows that $s \geq t/(10K^2)$.  This is because by the minimality of $s$, we have $N_t^X(i)\geq s/3$ for all players $X$ and all arms $i$ thanks to the uniform exploration during times in $[t_j, \max(s, 1.5t_j) )$.  Thus, the only way for the event $E$ to occur is if
\[
 \delta_{t_j} - \eps_{t/(20K^2)}\leq \gap_{P_*(\p)}(\p) \leq \delta_{t_j} + \eps_{t/(20K^2)} \,.
\]
Recalling that $\delta_{t_j}$ is chosen uniformly at random from an interval of length $0.5\eps_{t_j}$ completes the proof.
\end{proof}

\begin{lemma}\label{lem:cut-stop}
For any instance $\p$, player $X$ and timestep $t \in [t_j, t_{j+1})$, consider nodes $P,Q$ such that $Q$ is a child of $P$.  Then
\[
\Pr\left[ \substack{ \left \lvert \gap_{Q}(\q_t^X) - c(P) \cdot \range_P(\q_t^X)  \right \rvert  \leq 6K\eps_t  \\ \text{ and } \gap_{P_*(\q_t^X)}(\q_t^X) \leq \delta_t + \eps_t \text{ and }\Omega}  \bigg| \exists Y \in [m], s \text{ such that } P \preceq  \mcl{P}_{c, \delta_s, \eps_s}(\q_s^Y) \right] \leq \frac{20K\eps_{t/(10K^2)}}{ \range_P(\p)}
\]
\end{lemma}
\begin{proof}
% Again, we may assume that event $\Omega$ happens (at negligible cost to the probability on the RHS). 
Let $E'$ denote the event in question that
\[
 \left \lvert \gap_{Q}(\q_t^X) - c(P) \cdot \range_P(\q_t^X)  \right \rvert  \leq 6K\eps_t \quad
 \text{ and } 
 \quad 
 \gap_{P_*(\q_t^X)}(\q_t^X) \leq \delta_t + \eps_t 
 \quad
 \text{ and }
 \quad
 \Omega.
\]
We restrict below to the event that $\Omega$ holds. By Lemma~\ref{lem:induction-step}, conditioning also on the event that $P \preceq  \mcl{P}_{c, \delta_s, \eps_s}(\q_s^Y)$ for some $s, Y$ implies that all players are always at either the special cut $P_*$, an ancestor of $P$, or a descendant of $P$.  First, by Lemma~\ref{lem:stay-at-special}, if Line~\ref{line:m-cut} ever triggers before timestep $t/(10K^2)$ then $E'$ cannot hold.  Now let $s$ be the first timestep during which some player $Y$ is at a strict descendant of $P$, i.e.  $P \prec  \mcl{P}_{c, \delta_s, \eps_s}(\q_s^Y)$.  Then we must have
\[
|\gap_Q(\q_s^Y) - c(P) \cdot \range_P(\q_s^Y)| \geq 6\eps_s  \,.
\]
Using the minimality of $s$ and the assumption that event $\Omega$ holds, the event $E'$ can only occur if $s \geq t/(10K^2)$.  In particular, we have now shown that for the first $t/(10K^2)$ timesteps, all players must be at ancestors of $P$.  Since $\Omega$ holds, we must then have that 
\begin{equation}
\label{eq:last-inequality-above}
\left \lvert \gap_{Q}(\p) - c(P) \cdot \range_P(\p) \right \rvert \leq 10\eps_{t/(10K^2)} \,.
\end{equation}
Finally it remains to note that the choice of $c(P)$ is in fact \textit{independent} of the event that we are conditioning on.  This is because we are only conditioning on the event that some player reaches $P$ or one of its descendants in $\cT_{K,m}$. This event is purely determined by the values of the function $c$ at strict ancestors of $P$. Moreover \eqref{eq:last-inequality-above} can hold only if $c(P)$ is contained within a certain interval of length $\frac{20\eps_{t/(10K^2)}}{ \range_P(\p)}$. Recalling that $c(P)$ is uniform in $\left[0,\frac{1}{K}\right]$ completes the proof.
\end{proof}

Next, we prove that if $\Delta(\p) \geq \Delta_j$, then after phase $j$, all players will immediately trigger Line~\ref{line:m-cut} (and thus will play optimally).

\begin{lemma}
\label{lem:late-phase}
Let $\p$ be an instance with $\Delta(\p)=\Delta\geq \Delta_j$ and assume the event $\Omega$ holds. Then for all times $t\geq t_j$ and all players $X\in [m]$, we have 
\[
    \cP_{c,\eps_t,\delta_t}(\q_t^X)=\cP_*.
\]
\end{lemma}  

%\mscomment{I guess we just define the relation between $\Delta_j$ and $t_j$ so this works. How about $\Delta_j=10\eps_{t_j}$?}

\begin{proof}
Note that by definition, $\Delta_j \geq 10\eps_{t_j}$.  Suppose for sake of contradiction that $\cP_{c,\eps_t,\delta_t}(\q_t^X)\neq \cP_*$ for some $X$. This implies
\begin{equation}
\label{eq:blah}
    \gap_{P_*(\q_t^X)}(\q_t^X) \leq \delta_t+\eps_t\leq 3\eps_{t_j}
\end{equation}
and yet by assumption
\begin{equation}
\label{eq:blah2}
    \gap_{P_*(\p)}(\p)\geq \Delta_j.
\end{equation}
With $i_k$ the $k$-th best arm according to $\q_t^X$, and $a_k$ the $k$-th best arm according to $\p$, it follows that
\begin{align*}
    \q_t^X(i_m)
    &\geq 
    \min_{k\leq m} \q_t^X(i_k)
    \\
    &\geq \min_{k\leq m} \q_t^X(a_k)
    \\
    &\geq \p(a_m)-\eps_t.
\end{align*}
Here the last step follows because Lemma~\ref{lem:safe-with-margin} and the assumption that $\Omega$ holds together ensure that each $a_k$ remains well-explored.
Combining with \eqref{eq:blah} and \eqref{eq:blah2}, it follows that
\begin{align*}
    \q_t^X(i_{m+1})
    &\geq
    \q_t^X(i_{m})-\gap_{P_*(q_t^X)}(\q_t^X)
    \\
    &\geq 
    (\p(a_m)-\eps_t)
    -
    3\eps_{t_j}
    \\
    &\geq
    \p(a_{m+1})+(\Delta_j-3\eps_{t_j}-\eps_t)
    \\
    &\geq 
    \p(a_{m+1})+3\eps_{t_j}.
\end{align*}
It follows that there exists $1\leq k\leq m+1$ such that
\begin{equation}
\label{eq:ik-bad-estimate}
    \q_t^X(i_k)\geq \bp(i_k)+3\eps_{t_j},
\end{equation}
since if not we would have
\begin{align*}
    \p(a_{m+1})
    &\geq \min_{k\leq m+1} \p(i_k)
    \\
    &\geq \min_{k\leq m+1} \q_t^X(i_k)-3\eps_{t_j},
\end{align*}
contradicting what we just showed. Moreover combining with Lemma~\ref{lem:safe-with-margin} implies that $i_k\notin \{a_1,\dots,a_m\}$. Fix such a choice $1\leq k\leq m+1$ so that \eqref{eq:ik-bad-estimate} holds. Let $s$ be the first time that arm $i_k$ was rejected by player $X$. Since $\Omega$ holds, we find
\[
    |\q_s^X(i_k)-\p(i_k)|\leq \frac{\eps_s}{10K}
\]
and so $\eps_s\geq 30K\eps_{t_j}.$ The fact that $i_k$ was rejected by $X$ at time $s$ implies 
\[
    \q_s^X(i_k)\leq \min_{k_0\leq m}\q_s^X(a_{k_0})-6\eps_s\leq \p(a_m)-5\eps_s
\]
because the well-exploredness of the arms $a_1,\dots,a_m$ implies
\[
    |\q_s^X(a_{k_0})-\p(a_{k_0})|\leq \frac{\eps_s}{10K},\quad 1\leq k_0\leq m.
\]
Combining the inequalities above, we find
\begin{align*}
    \p(a_m)-5\eps_s
    &\geq
    \q_s^X(i_k)
    \\
    &\geq
    \q_t^X(i_k)-\frac{\eps_s}{2}
    \\
    &\geq
    \q_t^X(i_{m+1})-\frac{\eps_s}{2}
    \\
    &\geq
    \q_t^X(i_m)-3\eps_{t_j}-\frac{\eps_s}{2}
    \\
    &\geq
    \p(a_m)-4\eps_{t_j}-\frac{\eps_s}{2}.
\end{align*}
Rearranging, we find $\eps_s<\eps_{t_j}$, which contradicts $\eps_s\geq 30K\eps_{t_j}$ above. This completes the proof.
\end{proof}

Putting everything together, we can now bound the overall regret of our algorithm on any instance.

\begin{theorem}\label{thm:full-regret-bound}
For any instance $\p$ with gap at least $\Delta\geq \Delta_j$, the strategy described above satisfies
\[
    R_{T,\Delta}\leq O\left(mK^{9/2}\sqrt{t_j\log(KT)}+mK\sqrt{t_jt_{j-1}\log(KT)}+K\right).
\]
Moreover with probability at least $1-\frac{1}{T}$ the players never collide.
\end{theorem}

\begin{proof}
By Lemma~\ref{lem:induction-step}, if $\Omega$ holds then 
\[
    \cP_{c,\eps_t,\delta_t}(\q_t^X),\cP_{c,\eps_t,\delta_t}(\q_t^Y)
\]
are equal or adjacent vertices in $\cT_{K,m}$ for all players $X,Y$ and all $t$. By construction, this implies that there are no collisions as long as $\Omega$ holds, which has probability at least $1-\frac{1}{T}$ by Lemma~\ref{lem:Omega}.

Next fix $\bp$ with $\Delta(\bp)\geq \Delta_j$. The event that $\Omega$ is false contributes regret at most $\frac{KT}{T}=K$. Below we consider only the case that $\Omega$ holds. As this implies there are no collisions, we can essentially estimate the regret player-by-player. 

First, Lemma~\ref{lem:late-phase} implies that there is zero regret from times $t\geq t_j$ when $\Omega$ holds. Indeed, since $\cP_{c,\eps_t,\delta_t}(\q_t^X)=P_*$ for all players $X$, the players simply play the top $m$ arms according to some permutation for $t\geq t_j$. Similarly there is no regret from being in a leaf, as by Lemma~\ref{lem:safe-with-margin} all leaves assign a permutation of the top $m$ arms to the players. Below we control the main regret contributions, which come from the padding layers.

Fixing $j_0<j$, we analyze the regret due to times $t\in [t_{j_0},t_{j_0+1})$. By Lemma~\ref{lem:special-stop}, the expected number of times that line~\ref{line:m-cut-padding} triggers for a fixed player $X$ is at most
\[
    (t_{j_0+1}-t_{j_0})\cdot\frac{5 \eps_{t/(20K^2)}}{\eps_{t_{j_0}}}.
\]
Note that 
\[
    \frac{\eps_{t/(20K^2)}}{\eps_{t_{j_0}}}\leq 
    \sqrt{\frac{20K^2 t_{j_0}\log(KT)}{t}}.
\]
Multiplying by $m$ for the total number of players, we obtain the upper bound
\[
    O\left(Km\sqrt{\log(KT)}\right)\cdot \sum_{t=t_{j_0}}^{t_{j_0+1}}\frac{1}{\sqrt{t}} 
    \leq 
    O\left(mK \sqrt{t_{j_0}t_{j_0+1}\log(KT)}\right)
\]
for the total number of times that Line~\ref{line:m-cut-padding} triggers during $t\in [t_{j_0},t_{j_0+1})$. We will simply upper-bound the associated regret by $1$ each time. Moreover it is easy to see that 
\[
    \sum_{j_0<j}\sqrt{t_{j_0}t_{j_0+1}}\leq O(\sqrt{t_{j-1}t_j})
\]
as long as $t_i\geq 2t_{i-1}$ holds.  Note at the beginning of Section~\ref{sec:full-alg} we ensured that $\Delta_{i-1} \geq 2\Delta_i$ so this inequality indeed holds.   
%\mscomment{Do we assume this somewhere?}

Finally we consider the expected number of times that Line~\ref{line:newskel} triggers. $\Omega$ implies that the event
\[
    \exists Y \in [m], s \text{ such that } P \preceq  \mcl{P}_{c, \delta_s, \eps_s}(\q_s^Y)
\]
occurs for at most one $P_j\in\cT_{K,m}$ at each depth $j$. Conditioning on this $P_j$ if it exists, Lemma~\ref{lem:cut-stop} implies the probability that Line~\ref{line:newskel} triggers thanks to $c(P)$ for a given player $X$ at time $t$ is at most 
\[
    O\left(\frac{K\eps_{t/(10K^2)}}{\range_P(\p)}\right).
\]
Moreover the associated regret from such an event is at most $\range_P(\p)$. Since there are $m$ players and $K$ values of $j$, noting that 
\[
    \eps_{t/(10K^2)}\leq O\left(\sqrt{\frac{K^5\log(KT)}{t}}\right),
\]
the resulting regret total from times $t\leq t_j$ is at most
\[
    O\left(mK^{9/2}\sqrt{t_j\log(KT)}\right).
\]
%\mscomment{Did something more efficient than our previous paper happen, or is there a mistake? Somehow $11\to 9$ happened.}

Combining, the total expected regret is at most
\[
    O\left(mK^{9/2}\sqrt{t_j\log(KT)}+mK\sqrt{t_jt_{j-1}\log(KT)}+K\right)
\]
as desired.
% We show 
% \[
%     R_T(\bp)\leq \wt O\lt(\frac{1}{\Delta_j\cdot\Delta_{j-1}}\rt).
% \]
% In fact by Lemma~\ref{lem:large-gap-perfect} it suffices to show 
% \[
%     R_{t_i}(\bp)\leq 
%     \wt O\lt(\frac{\sqrt{t_i}}{\Delta_{j-1}}\rt)
%     \leq 
%     \wt O\lt(\frac{1}{\Delta_j\cdot\Delta_{j-1}}\rt).
% \]
% This follows from Lemma~\ref{lem:cost} just as in \cite[Proof of Theorem 3.2]{bubeck2021cooperative}. 
\end{proof}

Theorem~\ref{thm:main-upper} now follows immediately from Theorem~\ref{thm:full-regret-bound}.
\begin{proof}[Proof of Theorem~\ref{thm:main-upper}]
Plugging in the definitions for $t_j$ in terms of $\Delta_j$ given at the beginning of Section~\ref{sec:full-alg} and applying Theorem~\ref{thm:full-regret-bound} gives the desired bounds.
\end{proof}

%% file: full-paper.bbl
\newcommand{\etalchar}[1]{$^{#1}$}
\begin{thebibliography}{BBM{\etalchar{+}}17}

\bibitem[ALK20]{alatur2020multi}
Pragnya Alatur, Kfir~Y Levy, and Andreas Krause.
\newblock Multi-player bandits: The adversarial case.
\newblock {\em Journal of Machine Learning Research}, 21:77, 2020.

\bibitem[AM14]{AM14}
Orly Avner and Shie Mannor.
\newblock Concurrent bandits and cognitive radio networks.
\newblock In {\em Joint European Conference on Machine Learning and Knowledge
  Discovery in Databases}, pages 66--81. Springer, 2014.

\bibitem[BB20]{bubeck2020coordination}
S{\'e}bastien Bubeck and Thomas Budzinski.
\newblock Coordination without communication: optimal regret in two players
  multi-armed bandits.
\newblock In {\em Conference on Learning Theory}, pages 916--939. PMLR, 2020.

\bibitem[BBM{\etalchar{+}}17]{bonnefoi2017multi}
R{\'e}mi Bonnefoi, Lilian Besson, Christophe Moy, Emilie Kaufmann, and Jacques
  Palicot.
\newblock Multi-armed bandit learning in iot networks: Learning helps even in
  non-stationary settings.
\newblock In {\em International Conference on Cognitive Radio Oriented Wireless
  Networks}, pages 173--185. Springer, 2017.

\bibitem[BBS21]{bubeck2021cooperative}
S{\'e}bastien Bubeck, Thomas Budzinski, and Mark Sellke.
\newblock Cooperative and stochastic multi-player multi-armed bandit: Optimal
  regret with neither communication nor collisions.
\newblock In {\em Conference on Learning Theory}, pages 821--822. PMLR, 2021.

\bibitem[BK18]{besson2018multi}
Lilian Besson and Emilie Kaufmann.
\newblock Multi-player bandits revisited.
\newblock In {\em Algorithmic Learning Theory}, pages 56--92. PMLR, 2018.

\bibitem[BLPS20]{BLPS20}
S{\'e}bastien Bubeck, Yuanzhi Li, Yuval Peres, and Mark Sellke.
\newblock Non-stochastic multi-player multi-armed bandits: Optimal rate with
  collision information, sublinear without.
\newblock In {\em Conference on Learning Theory}, pages 961--987. PMLR, 2020.

\bibitem[BP19]{BP18}
Etienne Boursier and Vianney Perchet.
\newblock {SIC-MMAB: synchronisation involves communication in multiplayer
  multi-armed bandits}.
\newblock In {\em Advances in Neural Information Processing Systems}, pages
  12071--12080, 2019.

\bibitem[DM95]{davis1995elementary}
Burgess Davis and David McDonald.
\newblock An elementary proof of the local central limit theorem.
\newblock {\em Journal of Theoretical Probability}, 8(3):693--702, 1995.

\bibitem[HCT22]{huang2021towards}
Wei Huang, Richard Combes, and Cindy Trinh.
\newblock Towards optimal algorithms for multi-player bandits without collision
  sensing information.
\newblock In {\em Conference on Learning Theory}, 2022.

\bibitem[LM21]{lugosi2021multiplayer}
G{\'a}bor Lugosi and Abbas Mehrabian.
\newblock Multiplayer bandits without observing collision information.
\newblock {\em Mathematics of Operations Research}, 2021.

\bibitem[MU17]{mitzenmacher2017probability}
Michael Mitzenmacher and Eli Upfal.
\newblock {\em Probability and computing: Randomization and probabilistic
  techniques in algorithms and data analysis}.
\newblock {Cambridge University Press}, 2017.

\bibitem[PBJ21]{pacchiano2021instance}
Aldo Pacchiano, Peter Bartlett, and Michael~I Jordan.
\newblock An instance-dependent analysis for the cooperative multi-player
  multi-armed bandit.
\newblock {\em arXiv preprint arXiv:2111.04873}, 2021.

\bibitem[RSS16]{RSS16}
Jonathan Rosenski, Ohad Shamir, and Liran Szlak.
\newblock Multi-player bandits--a musical chairs approach.
\newblock In {\em International Conference on Machine Learning}, pages
  155--163. PMLR, 2016.

\bibitem[SXSY20]{shi2020decentralized}
Chengshuai Shi, Wei Xiong, Cong Shen, and Jing Yang.
\newblock Decentralized multi-player multi-armed bandits with no collision
  information.
\newblock In {\em International Conference on Artificial Intelligence and
  Statistics}, pages 1519--1528. PMLR, 2020.

\end{thebibliography}
